\documentclass[aps,pra,reprint,superscriptaddress]{revtex4-2}

\usepackage{amsmath}
\usepackage{amssymb}
\usepackage{graphicx}
\usepackage{amsthm}
\usepackage{mathtools}
\usepackage{bm}
\usepackage{xcolor}
\usepackage{booktabs}   
\usepackage{tabularx}   
\usepackage{array}      

\newcolumntype{L}{>{\raggedright\arraybackslash}X}
\newcolumntype{C}{>{\centering\arraybackslash}X}

\theoremstyle{plain}
\newtheorem{thm}{Theorem}[section]

\newtheorem{lem}[thm]{Lemma}

\theoremstyle{definition}

\theoremstyle{remark}

\newenvironment{proofsketch}{
  \begin{proof}
}{\end{proof}}

\begin{document}

\title{Mitigating Barren Plateaus in Quantum Denoising Diffusion Probabilistic Models}

\author{Haipeng Cao}
\thanks{The authors contributed equally to this work.}
\affiliation{School of Computer Science and Technology, University of Science and Technology of China, Hefei, Anhui, China}
\affiliation{International Quantum Academy, Shenzhen, 518048, China}

\author{Kaining Zhang}
\thanks{The authors contributed equally to this work.}
\affiliation{College of Computing and Data Science, Nanyang Technological University, Singapore, Singapore}

\author{Dacheng Tao}
\affiliation{College of Computing and Data Science, Nanyang Technological University, Singapore, Singapore}

\author{Zhaofeng Su}
\email{youngpath2012@gmail.com}
\affiliation{Research Institute of Quantum Technology, The Hong Kong Polytechnic University, Hung Hom, Hong Kong, China}
\affiliation{School of Computer Science and Technology, University of Science and Technology of China, Hefei, Anhui, China}

\date{\today}

\begin{abstract}
Quantum generative models exploit quantum superposition and entanglement to enhance learning efficiency for both classical and quantum data. Recently, inspired by classical diffusion frameworks, the quantum denoising diffusion probabilistic model has emerged as a powerful tool for learning correlated noise models, many-body phases, and topological data structures. However, we demonstrate that this framework is currently restricted to small-scale systems. As the system size increases, a severe barren plateau problem emerges, fundamentally limiting the model's scalability. We provide rigorous theoretical proofs and experimental validation to identify the origin of this barren plateau, distinct from previously known causes. To restore trainability, we introduce an enhanced architecture that effectively mitigates the barren plateau phenomenon and guarantees the model's trainability in the tested settings. Building on this architecture, we further propose a conditional quantum denoising diffusion probabilistic model, capable of generating ground states based on Hamiltonian parameters, expanding the utility of quantum generative models for complex quantum state preparation to a certain extent. Our approach not only holds the potential to address the scalability and trainability bottlenecks of quantum diffusion models, but also provides a robust tool for exploring complex quantum matter and state preparation in the NISQ era.
\end{abstract}

\maketitle

\section{Introduction}
Quantum computing is a new computing paradigm based on quantum mechanics~\cite{NL00}, which has a significant advantage in computing speed over its classical counterpart. The advantage originates from the distinctive phenomena such as entanglement and nonlocality of quantum mechanics~\cite{ZhaofengQIP2018, ZhaofengPRA2020}, and has been theoretically convinced by novel quantum algorithms and protocols~\cite{Shor94, Grover96, ZhaofengPRA2018}. We have now entered the Noisy Intermediate-Scale Quantum (NISQ) era, characterized by devices with dozens to hundreds of physical qubits~\cite{Preskill2018}. NISQ devices have physically demonstrated the advantages of quantum computing in both superconducting and photon approaches~\cite{GoogleSycamore2019, PanScience2020}. However, the inherent noise in NISQ devices remains a primary obstacle to the reliability of the large-scale quantum circuits and thus limits the scope of practical applications. 

Interestingly, classical machine learning (ML) often demonstrates resilience to noise. For instance, controlled noise in stochastic gradient descent can assist in escaping sharp local minima and saddle points~\cite{Chi-JACM2021}. In the quantum domain, classical ML has already provided transformative solutions for many challenging tasks such as quantum state tomography, phase transition analysis, preparation of complex quantum states, and the efficient representation of many-body wavefunctions~\cite{CML2mp1,CML2mp2,CML2mp3,wf1,QST1,QST2}. However, classical approaches are fundamentally constrained by the "curse of dimensionality", as the dimension of Hilbert space grows exponentially with the number of qubits~\cite{wf1, huang2022provably}. 

Quantum Machine Learning (QML) addresses these limitations by integrating quantum circuits into ML framework to harness the high-dimensional expressivity, representing a highly promising frontier~\cite{Wang_RPP2024,Sweke2020,Cerezo-NCS2022}. 
Significant progress has been made in extending classical models to the quantum setting, including quantum support vector machines~\cite{Patrick-QSVM2014}, quantum principal component analysis~\cite{Lloyd-NP2014}, quantum neural networks~\cite{Kerstin-NC2020}, and the robustness verification of QML models~\cite{ji2024robustness}.
Most current QML research operates within the framework of variational quantum algorithms~\cite{cerezo2021variational}, which utilize hybrid quantum-classical optimization loops to address challenging problems in quantum chemistry~\cite{Alberto-NC2014} and combinatorial optimization\cite{farhi-QAOA2014}. Within the subset of generative modeling, several architectures have been explored to learn and reproduce quantum data distributions, including quantum generative adversarial networks(QuGANs)~\cite{QuGAN, L.M.Duan-GAN-2018}, quantum variational autoencoders (QVAEs) \cite{QVAE}, quantum circuit Born machines (QCBMs)~\cite{QCBM_1,QCBM_2}. While these models offer significant promise for state preparation, they often suffer from training instabilities, such as the mode collapse observed in QuGANs~\cite{EQuGAN}. 

Classical diffusion models (CDMs) constitute one of the core architectures in modern machine learning. Quantum diffusion models (QDMs) extend these models to the quantum domain. Existing QDMs can be broadly categorized into two categories according to whether the diffusion process operates on classical or quantum data. The first is Classical-Quantum Noise-Guided Diffusion Models (CQGDMs), characterized by a classical forward diffusion process and a reverse denoising process governed by Quantum Neural Networks (QNNs) or parameterized quantum circuits (PQCs) \cite{QNDGDM}. Prominent architectures within this paradigm include Quantum Diffusion Model (this specifically refers to a model named Quantum Diffusion Model, rather than the category of QDMs) \cite{QDM1}, Optimized Quantum Implicit Denoising Diffusion Model \cite{OQIDDM}, quantum dense model, and quantum U-Net \cite{Q-Dense--U-Net}. Since the training datasets for this framework are strictly classical and the generated data are ultimately mapped back to classical probability distributions, CQGDMs are conventionally employed for classical data generation \cite{QNDGDM}. While both CQGDMs and CDMs target classical data generation, experiments demonstrate that CQGDMs offer superior parameter efficiency and require fewer denoising steps to achieve equivalent data quality \cite{QDM1}. Furthermore, the inherent advantages of quantum state representations grant CQGDMs an unparalleled edge over CDMs when tackling ultra-high-dimensional data \cite{QDM1}. The other category is Quantum-Quantum Noise-Guided Diffusion Model (QQGDMs), in which the forward process is implemented by quantum scrambling circuits (QSCs) or through other quantum means, while the reverse process is comprised of GNNs or PQCs \cite{QNDGDM}. Representative instances of this category encompass quantum denoising diffusion probabilistic model (QuDDPM) \cite{QuDDPM}, quantum generative diffusion model \cite{QGDM-QQ}, mixed-state quantum denoising diffusion probabilistic model \cite{MSQuDDPM}, and Temporal-aware QuDDPM \cite{TQuDDPM}. Composed entirely of quantum circuits, QQGDMs are capable of efficiently learning and generating quantum data, a task that remains exceptionally formidable for CDMs or classical computers.

The QuDDPM serves as a representative framework within the family of QQGDMs. By adapting the hierarchical denoising logic of classical diffusion models, QuDDPM has demonstrated an unprecedented ability to learn correlated noise models and reconstruct complex many-body phases. However, the scalability of QuDDPM is hindered by a critical bottleneck: the Barren Plateau (BP) phenomenon. In a BP, the variance of loss function gradient vanishes exponentially with the system size, rendering the landscape flat and the model untrainable~\cite{0Barren}. While existing literature has identified several triggers for BPs~\cite{0Barren,BP_Entanglement_1,BP_Entanglement_2,BP_Noise} and proposed solutions~\cite{init_BP,NEURIPS2022_7611a3cb,ansatz_BP}, the underlying mechanism of BPs in the diffusion-based quantum models remains largely unexplored. 

\begin{figure*}[t]
    \vskip 0.2in
    \begin{center}
    \centerline{\includegraphics[width=1.6\columnwidth]{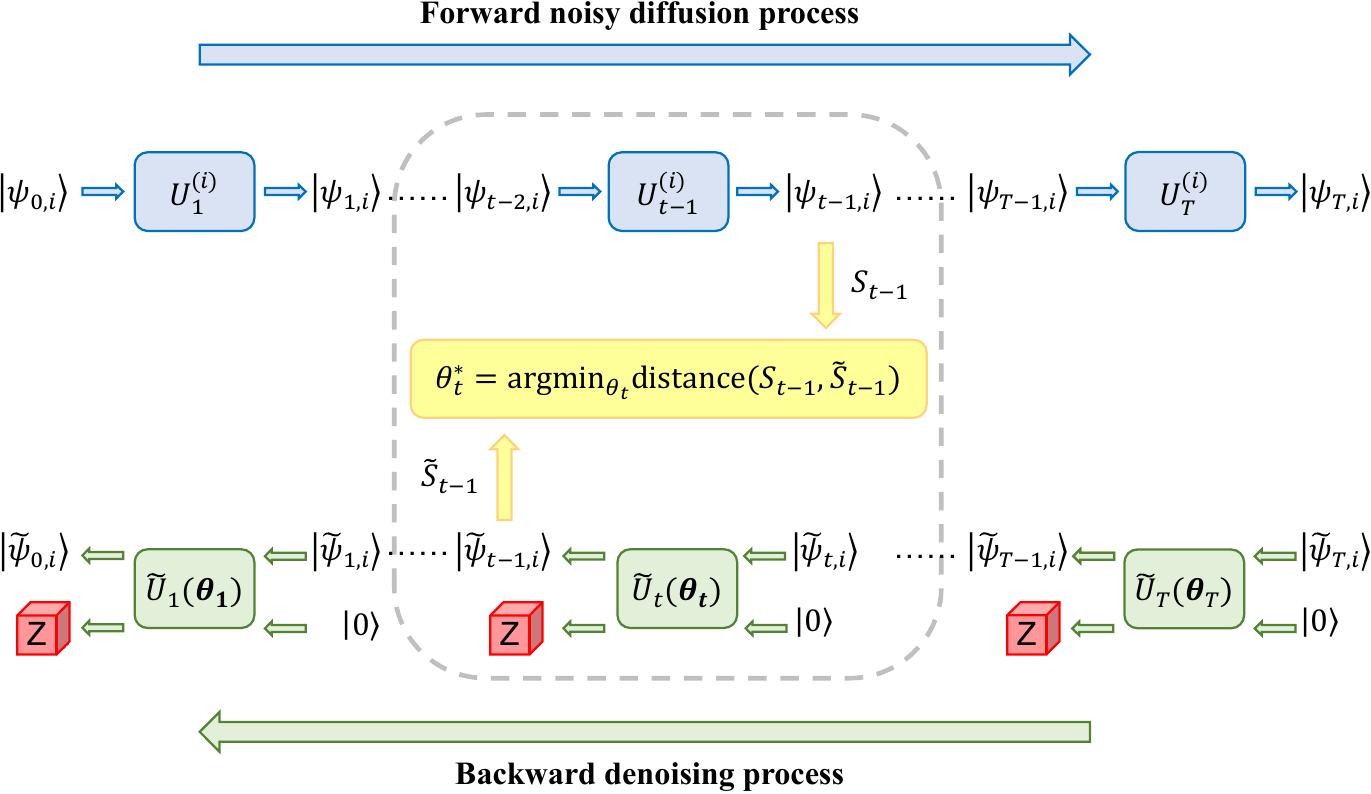}}
    \caption{ Structure of QuDDPM. The top part of the figure shows the forward noisy diffusion process and the bottom part shows the backward denoising process. }
    \label{fig:structure of QuDDPM}
    \end{center}
    \vskip -0.2in
\end{figure*}

In this work, we systematically investigate the BP phenomenon in QuDDPM. We provide both theoretical and numerical evidence for the existence of BPs in QuDDPM. Our analysis further shows that this phenomenon is induced by Haar-random states, revealing a previously unrecognized origin of BPs. Based on this insight, we propose a mitigation strategy that introduces an auxiliary qubit system to perturb the generated states and guide them toward the target state. In the experiments, this approach enables the model to effectively mitigate BPs and improves trainability compared with the original QuDDPM.

Furthermore, we extend the model to a \textit{conditional QuDDPM}, in which the parameters of a given Hamiltonian are provided as input and the corresponding ground state is generated. This extension broadens the applicability of the model to quantum many-body problems. The numerical results for the conditional QuDDPM are validated through two-point correlation functions and classification of matter phases. The results show that the generated states accurately compute the corresponding two-point correlation functions and achieve classification of matter phases, demonstrating the strong expressive power of the improved model and its effectiveness for quantum applications.

\section{Overview of Q\NoCaseChange{u}DDPM}

The recently proposed QuDDPM is the state-of-the-art quantum generation model, which was inspired by classical DDPMs. QuDDPM aims to generate new samples by learning from an ensemble of samples in the target distribution. Specifically, it is known that we are given a set of quantum states $\mathcal{S} = \{ |\psi_{k} \rangle\} \sim \varepsilon_0$ sampled from an unknown distribution $\varepsilon_0$, and the task of the model is to learn from $\mathcal{S}$ such that it can generate states that conform to the distribution $\varepsilon_0$. This quantum state set $\mathcal{S}$ is the training set of the model. This is accomplished by feeding samples from the target distribution into the forward diffusion process to add noise until they become Haar-random states, then training the backward denoising process to invert the forward diffusion. Once training is complete, new samples can be generated by implementing the trained backward denoising process. Next, we present the details of QuDDPM.

\subsection{Structure of QuDDPM} 
The architecture of QuDDPM is shown in Fig.~\ref{fig:structure of QuDDPM}. QuDDPM consists of two main components: the forward diffusion process and the reverse denoising process, each comprising T timesteps. In the forward process, the training set $ \mathcal{S} $ is taken as input; at each timestep, a QSC adds random noise to each state in $ \mathcal{S} $, thereby producing a set of intermediate states. The distributions corresponding to these intermediate states serve as the intermediate target distributions for the corresponding timestep in the reverse process. After T timesteps of the forward process, the resulting states closely approximate the Haar distribution. The theoretical guarantee for this convergence to the approximate Haar regime lies in the property that if the QSC can approximate a quantum $K$-design, the deviation $E_{\text{diff}}$ between the final states generated by the forward diffusion process and the ideal Haar-random states satisfies $E_{\text{diff}} \sim \mathcal{O}(e^{-T})$. That is, the deviation $E_{\text{diff}}$ decays exponentially with the increase of the total timesteps $T$. Physically, the QSC consists of alternating single-qubit rotation gates and all-to-all $ZZ$ entangling layers (i.e. Fig.~\ref{fig:circuit}), this fast-scrambling architecture is capable of rapidly diffusing local information across the entire Hilbert space. Consequently, for a forward process with a total circuit depth scaling as $\Omega(n_{data})$, where $n_{data}$ is the number of data qubits, which corresponds to the number of qubits of the quantum states learned and generated by QuDDPM, the resulting final states are guaranteed to conform to the approximate Haar distribution \cite{QuDDPM}.

The backward process takes Haar-random states as input and aims to remove noise from these states through T PQCs, such that the finally generated states conform to the target distribution. Each timestep in the backward process is also referred to as a training cycle, with each cycle corresponding to one PQC. By learning the intermediate states generated in the forward process, the model optimizes the parameters of the PQC so that the states processed by the PQC match the corresponding intermediate target distributions. Haar-random states can be obtained by applying a sequence of QSCs to $\vert{}0\rangle^{\otimes n_{data}}$.

\subsection{Quantum circuit structure} 
Fig.~\ref{fig:circuit} illustrates the structure of the quantum circuits involved in QuDDPM. Subfigure (a) shows the circuit structure of the QSC, which generates different random parameters for each state in the state set, thereby adding noise to achieve the diffusion effect. Subfigure (b) illustrates the circuit structure of the PQC. It is noted that, in addition to the $n_{data}$ data qubits, the PQC also processes $n_A$ ancilla qubits. After the $(n_{data} + n_A)$ qubits are collectively processed by the PQC, a measurement operation is performed on the $n_A$ ancilla qubits to introduce non-linear operations and enhance the expressivity of the circuit. The resulting state of the $n_{data}$ data qubits at this stage represents the output state produced by the current PQC processing.

\begin{figure*}
    \vskip 0.2in
    \begin{center}
    \centerline{\includegraphics[width=1.8\columnwidth]{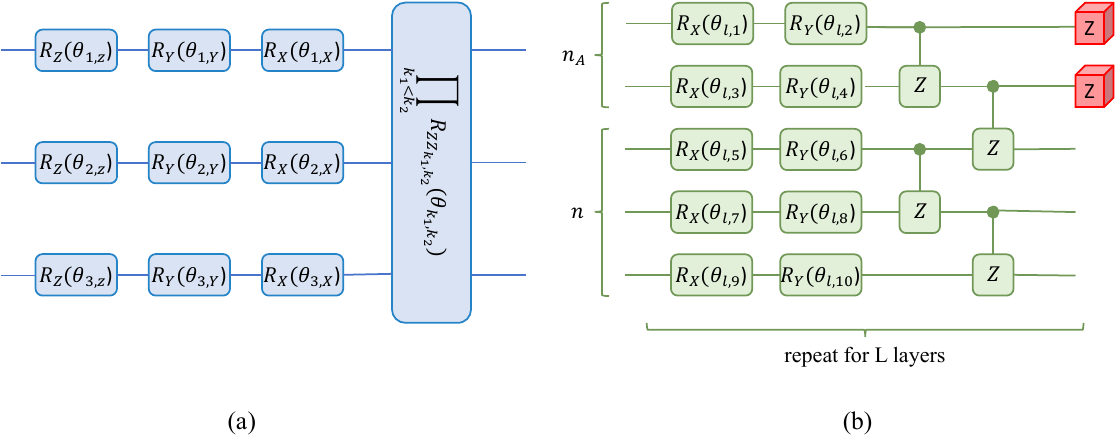}}
    \caption{ Quantum circuit structure. (a) is the circuit of one step of the forward diffusion process on a system of $n=3$ qubits. (b) is one-layer architecture of L-layer of $U_t(\bm \theta_t)$ on a system of $n=3$ data and $n_A=2$ ancilla qubits. }
    \label{fig:circuit}
    \end{center}
    \vskip -0.2in
\end{figure*}

\subsection{Single-Cycle Local Structure and Parameter Optimization Mechanism} 
As shown in the dashed box of Fig.~\ref{fig:structure of QuDDPM}, this part provides the details of training cycle t (with training cycles labeled from T down to 1). This training cycle takes as input the output from the previous cycle, that is, the output of PQC $\tilde{U}_{t+1}(\bm{\theta}_{t+1})$, and processes each state $|\tilde{\psi}_{t,i}\rangle$ in the input state set $\tilde{\mathcal{S}}_{t}$ via PQC $\tilde{U}_{t}(\bm{\theta}_{t})$ to obtain the output state set $\tilde{\mathcal{S}}_{t-1}$. By computing and minimizing the distance between $\tilde{\mathcal{S}}_{t-1}$ and the intermediate state set $\mathcal{S}_{t-1}$ generated during the forward process, the parameters of $\tilde{U}_{t}(\bm{\theta}_{t})$ are optimized so that the states processed by $\tilde{U}_{t}(\bm{\theta}_{t})$ conform as closely as possible to the intermediate target distribution corresponding to $\mathcal{S}_{t-1}$. Repeating this procedure T times, the state set produced by the output of $\tilde{U}_{1}(\bm{\theta}_{1})$ in the final training cycle constitutes the set that conforms to the target distribution. Once training is completed, a Haar-random state can be fed into the backward process to generate states that follow the target distribution. Notably, since the generated data set$\tilde{\mathcal{S}}_{t-1}$ may be smaller than the real data $\mathcal{S}_{t-1}$, the distance calculated between them is typically the distance between $\tilde{\mathcal{S}}_{t-1}$ and a subset $\mathcal{S}^{'}_{t-1}$ of $\mathcal{S}_{t-1}$, which has the same size as  $\tilde{\mathcal{S}}_{t-1}$.

To enhance clarity and better distinguish between the various set notations introduced in this work, we provide a collective overview and further clarification below:
\begin{itemize}
    \item $\mathcal{S}_{t-1}$: The state ensemble obtained after applying $t-1$ QSCs to the training set $\mathcal{S}_{0}$ during the forward process. It serves as the target state set when training the PQC $\tilde{U}_{t}(\bm{\theta}_{t})$ in the $(T-t+1)$-th training cycle of the reverse process.
    \item $\tilde{\mathcal{S}}_{t-1}$: The set of quantum states generated by the PQC $\tilde{U}_{t}(\bm{\theta}_{t})$.
    \item $\mathcal{S}^{'}_{t-1}$: A subset of $\mathcal{S}_{t-1}$. This is because the purpose of training the PQC $\tilde{U}_{t}(\bm{\theta}_{t})$ is to make the distribution corresponding to $\tilde{\mathcal{S}}_{t-1}$ as close as possible to that of $\mathcal{S}_{t-1}$, which requires measuring the distance between $\tilde{\mathcal{S}}_{t-1}$ and $\mathcal{S}_{t-1}$. Since $\mathcal{S}_{t-1}$ is often larger than $\tilde{\mathcal{S}}_{t-1}$, we select a subset $\mathcal{S}^{'}_{t-1}$ of $\mathcal{S}_{t-1}$ of the same size as $\tilde{\mathcal{S}}_{t-1}$, and measure the distance between $\mathcal{S}^{'}_{t-1}$ and $\tilde{\mathcal{S}}_{t-1}$ to evaluate the similarity between the distributions corresponding to $\tilde{\mathcal{S}}_{t-1}$ and $\mathcal{S}_{t-1}$.
    \item $\vert{}\mathcal{S}\vert{}$: A notation that appears in the subsequent sections, representing the size of the sets $\tilde{\mathcal{S}}_{t-1}$ and $\mathcal{S}^{'}_{t-1}$ (i.e., $\vert{}\mathcal{S}\vert{} = \vert{}\mathcal{S}^{'}_{t-1}\vert{} = \vert{}\tilde{\mathcal{S}}_{t-1}\vert{}$). Note that the symbol $\mathcal{S}$ itself carries no independent physical meaning.
\end{itemize}

\subsection{PQC Ansatz and Variational Training Protocol} 
To facilitate the subsequent analysis of the BP problem, we present here the general structure of the PQC: 

\begin{equation}  \label{eq:U_t}
    \tilde{U}_t(\bm \theta_t)=\prod\limits_{l=1}^{L}W_t V_t(\bm \theta_{t,l}),
\end{equation}
where $W_t$ is a generic unitary operator that does not depend on any angle $\theta$, $V(\bm \theta_{t,l})={} \otimes_{\lambda=0}^{n-1} \prod\limits_{i=0}^{\tau-1}R_{\sigma^{(i)}}(\theta_{t,l,\lambda\tau+i})$, $\sigma^{(i)} \in \{\sigma_1, \sigma_2, \sigma_3\}$, $n = n_{data} + n_A$.

Additionally, the loss function used by the model is given by  

\begin{equation} \label{eq:loss}
    \begin{split}
        \mathcal{L}_t(\bm \theta_t) &=\mathcal{D}_{\rm MMD}(\tilde{\mathcal{S}}_{t-1}, \mathcal{S}^{'}_{t-1}) \\ &=\mathcal{D}_{\rm MMD}(\tilde{U}_t(\bm \theta_t)\tilde{\mathcal{S}}_t, \mathcal{S}^{'}_{t-1}),
    \end{split}
\end{equation}
where $\mathcal{D}_{\rm MMD}(\mathcal{S}_1, \mathcal{S}_2)=\bar{F}(\mathcal{S}_1, \mathcal{S}_1)+\bar{F}(\mathcal{S}_2, \mathcal{S}_2)-2\bar{F}(\mathcal{S}_1, \\ \mathcal{S}_2)$, $\bar{F}(\mathcal{S}_1, \mathcal{S}_2)=\mathbb{E}_{|\phi\rangle \in \mathcal{S}_1, |\psi\rangle\in\mathcal{S}_2}[|\langle\phi|\psi\rangle|^2]$, both $\mathcal{S}_1$ and $\mathcal{S}_2$ are quantum state ensembles.

\section{THE barren plateau phenomenon in QuDDPM}

As described earlier, the BP is a phenomenon where the gradient decays exponentially with the number of qubits, making it difficult for the ML model to handle large-scale problems. \textbf{Existing studies show that BPs are mathematically defined by a zero mean gradient and a variance decaying exponentially with the number of qubits \cite{0Barren}}. We note that the statistical properties of Haar-random states seem to induce the above behavior in the model. Therefore, we conduct the following theoretical analysis on the potential influence of Haar-random states on the gradient of QuDDPM.

For the analysis of the BP, we focus on the loss function Eq.~\ref{eq:loss}. The Eq.~\ref{eq:loss} can be viewed as a function of the unitary operator $\tilde{U}_t(\bm \theta_t)$, so we need to first compute the derivative of $\tilde{U}_t(\bm \theta_t)$ with respect to $\bm \theta_t$ before calculating the gradients of the loss. Therefore, we calculate the gradient of the unitary operator $\tilde{U}_t(\bm \theta_t)$ with respect to the parameter $\theta_{t,l,k}$ as follows:

\begin{equation} \label{eq:U_partial_derivative}
    \begin{split}
        \partial_{l,k}\tilde{U}_t({\bm \theta_t})=-\frac{i}{2}\tilde{U}_{t,L:l}K\tilde{U}_{t,l-1:1},
    \end{split}
\end{equation}
where $\alpha=k\mod \tau$, $K=(\otimes_{\lambda=0}^{n-1}\prod\limits_{i=\alpha+1}^{\tau-1}R_{\sigma^{(i)}}(\theta_{t,l,\lambda\tau+i}))^\dagger [ \\ (I^{0:\lfloor k/\tau\rfloor}\otimes \sigma^{(\alpha)} \otimes I^{\lfloor k/\tau+1\rfloor:n})(\otimes_{\lambda=0}^{n-1}\prod\limits_{i=\alpha+1}^{\tau-1}R_{\sigma^{(i)}}(\theta_{t,l,\lambda\tau+i}))$, $I^{i:j}$ denotes that the operator $I$ acts on the qubits in the range $[i,j)$. Obviously, $K$ is a Hermitian operator. Furthermore, for $\tilde{U}_t({\bm \theta_t})$ with multiple layers, we define $\tilde{U}_{t,l_1:l_2}=\prod\limits_{l=l_2}^{l_1}W_t V_t(\bm \theta_{t,l})$. The proof of Eq.~\ref{eq:U_partial_derivative} is provided in Appendix \ref{ap:U_gradient}.

Based on the results of Eq.~\ref{eq:U_partial_derivative}, we formally compute the gradient of the loss function Eq.~\ref{eq:loss}. We note that $\mathcal{L}_t(\bm \theta_t) = \bar{F}(\mathcal{S}_{t-1}^{'},\mathcal{S}_{t-1}^{'}) + \bar{F}(\tilde{S}_{t-1},\tilde{S}_{t-1}) - 2\bar{F}(\tilde{S}_{t-1},S_{t-1}^{'})$. Therefore, to present the results more clearly, we compute the gradient of each of the three terms individually. Through straightforward calculation, it is not difficult to obtain $\partial_{l,k}\bar{F}(\mathcal{S}_{t-1}^{'}, \mathcal{S}_{t-1}^{'})=0$ and $\partial_{l,k}\bar{F}(\tilde{S}_{t-1},\tilde{S}_{t-1})=0$. It is evident that the gradient of the loss function Eq.~\ref{eq:loss} depends only on $\bar{F}(\tilde{S}_{t-1},S_{t-1}^{'})$. Consequently, we have
\begin{equation} \label{eq:loss_partial_derivative}
    \begin{split}
        \partial_{l,k}\mathcal{L}_t =&\frac{i}{|\mathcal{S}|^2}\sum\limits_{i,j}{\rm Tr}(|\psi_{t-1,j}\rangle\langle\psi_{t-1,j}|\cdot\tilde{U}_{t,L:l}^{'}[K, \\
        &\tilde{U}_{t,l-1:1}^{'}|\tilde{\psi}_{t,i}\rangle\langle\tilde{\psi}_{t,i}|\tilde{U}_{t,l-1:1}^{'\dagger}]\tilde{U}_{t,L:l}^{'\dagger}),
    \end{split}
\end{equation}
where $|\mathcal{S}|=|\mathcal{S}_{t-1}|=|\tilde{\mathcal{S}}_{t-1}|$, $|\psi_{t-1,j}\rangle\in\mathcal{S}_{t-1}$, $|\tilde{\psi}_{t,i}\rangle\in\tilde{\mathcal{S}}_t$, $\tilde{U}^{'}_{t}$ represents the part of $\tilde{U}_{t}$ applied to the data qubits. The proof is provided in Appendix \ref{ap:loss_gradient}.

As mentioned above, we suspect that the BP observed in QuDDPM arises from the statistical properties of Haar-random states, which cause the mean of the gradient to be zero and the variance to decay exponentially with the number of qubits. To prove this hypothesis, we redefine the gradient of loss function for the scenario where the input to the PQC is Haar-random states, as follows:
\begin{equation} \label{eq:L_haar_p}
\begin{split}
    \partial_{l,k}\mathcal{L}_H=&\frac{i}{|\mathcal{S}|^2}\sum\limits_{i,j}{\rm Tr}(|\psi_{j}\rangle\langle\psi_{j}|\cdot\tilde{U}_{L:l}^{'}[K,\tilde{U}_{l-1:1}^{'} \cdot \\ & U_{Hi}|0\rangle\langle0|U_{Hi}^\dagger \cdot\tilde{U}_{l-1:1}^{'\dagger}]\tilde{U}_{L:l}^{'\dagger}),
\end{split}
\end{equation}
where $U_{Hi}\in\mathcal{U}_H(N)$ obeying Haar measure, $N=2^{n_{data}}$; the gradient of loss function when the input to the PQC is approximate Haar-random states as follows:
\begin{equation} \label{eq:L_ahaar_p}
\begin{split}
    \partial_{l,k}\mathcal{L}_A=& \frac{i}{|\mathcal{S}|^2}\sum\limits_{i,j}{\rm Tr}(|\psi_{j}\rangle\langle\psi_{j}|\cdot\tilde{U}_{L:l}^{'}[K,\tilde{U}_{l-1:1}^{'}\cdot \\ & U_{Ai}|0\rangle\langle0|U_{Ai}^\dagger \cdot\tilde{U}_{l-1:1}^{'\dagger}]\tilde{U}_{L:l}^{'\dagger}),
\end{split}
\end{equation}
where $U_{Ai}\in\mathcal{U}_A(N)$ obeying approximate Haar measure. 

To further proceed the analysis of BPs phenomenon in QuDDPM, we need Lemma~\ref{lem:1} and Lemma~\ref{lem:2} as follows. For the proofs of the two lemmas, please refer to the Appendix \ref{ap:lem1} and \ref{ap:lem2}.

\begin{lem} \label{lem:1}
    Assuming A, B, C, U and V are all N-dimensional unitary matrices, then
    \begin{equation}
        \begin{split}
            &{\rm Tr}(AB[V,CU|0\rangle \langle 0|U^\dagger C^\dagger]B^\dagger) \\
            =&\sum\limits_{b,b',e,e'<N} D_{ib} U_{b1} U_{1e}^\dagger E_{ei} - F_{jb' }U_{b'1} U_{1e'}^\dagger G_{e'j},
        \end{split}
    \end{equation}
    where$D=ABVC$, $E=C^\dagger B^\dagger$, $F=ABC$, $G=C^\dagger V B^\dagger$.
\end{lem}

\begin{lem} \label{lem:2}
    Assuming A, B, C, D and U are all N-dimensional unitary matrices, and $U \in \mathcal{U}(N)$ match the Haar distribution. Then, we have
    \begin{equation}
        \begin{split}
            &\mathbb{E}_{U \in \mathcal{U}(N)}(\sum A_{ae}U_{e1}U_{1f}^\dagger B_{fa}C_{bg}U_{g1}U_{1h}^\dagger D_{hb}) \\
            &\in [-\frac{2}{N^2-1}, \frac{2}{N^2-1}].
        \end{split}
    \end{equation}
\end{lem}

After giving the above two lemmas, we present the core theorem of this paper: 

\begin{thm} \label{thm:bp}
For Eq.~\ref{eq:L_haar_p} and Eq.~\ref{eq:L_ahaar_p}, if $U_{Hi}\in\mathcal{U}_H(N)$ obeying Haar measure, $U_{Ai}\in\mathcal{U}_A(N)$ obeying approximate Haar measure, then:
\begin{gather}
    \langle \partial_{l,k} \mathcal{L}_H \rangle = 0, \\
    |\mathrm{Var}(\partial_{l,k} \mathcal{L}_H)| \leq \frac{8}{|\mathcal{S}|^4 \cdot (2^{2n_{data}}-1)}, \\
    \langle \partial_{l,k} \mathcal{L}_A \rangle = \varepsilon, \\
    |\mathrm{Var}(\partial_{l,k} \mathcal{L}_A)| \leq \frac{8}{|\mathcal{S}|^4 \cdot (2^{2n_{data}}-1)} + 4\zeta.
\end{gather}
where $n_{data}$ is the number of data qubits, $\varepsilon, \zeta$ denote the deviations of approximate Haar-random states from exact Haar-random states at different angles. Specifically, $\epsilon$ quantifies the trace-distance error between the approximate and exact Haar-random states, whereas $\zeta$ measures the deviation of the 4th-order tensor integration between approximate and exact Haar unitary operators. Physically, $\zeta$ can be characterized as an algebraic variant of the diamond norm between the actual 2-fold channel used to generate approximate Haar states and the ideal Haar twirling channel.
\end{thm}

\begin{proofsketch}
The proof of the Theorem \ref{thm:bp} consists of two core parts: the statistical analysis under the exact Haar measure and the approximate Haar measure case. The key idea is to leverage the properties of the unitary group matching the Haar distribution to convert complex quantum operator integrals into algebraic operations, thereby quantifying the distribution characteristics of the loss function gradient in Hilbert space.

\textbf{Exact Haar measure (Eq. \ref{eq:L_haar_p})}: 
\begin{itemize}
    \item \textbf{Mean:} By leveraging the 1-design property of the Haar measure \cite{theory_unitart_group},
        \begin{equation} \label{eq:unitary_2}
           \int_{\mathcal{U}_H(N)}u_{ij}\bar{u}_{i'j'}dU_H=\frac{1}{N}\delta_{ii'}\delta_{jj'},
        \end{equation}
        where $u_{ij}$ and $u_{i'j'}$ are the element of unitary operator $U_H$, the expression for the expectation value is transformed into $\langle\partial_{l,k}\mathcal{L}_H\rangle=\frac{i}{|\mathcal{S}|^2}\sum\limits_{i,j}{\rm Tr}(|\psi_{j}\rangle\langle\psi_{j}|\cdot\tilde{U}_{L:l}^{'}[K,\frac{I}{2^{n_{data}}}]\tilde{U}_{L:l}^{'\dagger})$. Due to the presence of the commutator $[K,\frac{I}{2^{n_{data}}}]$, it follows that $\langle\partial_{l,k}\mathcal{L}_H\rangle=0$.

    \item \textbf{Variance:} By leveraging the 2-design properties of the unitary group \cite{theory_unitart_group},
        \begin{equation} \label{eq:unitary_4}
            \begin{split}
                &\int_{\mathcal{U}_H(N)} u_{i_1j_1}u_{i_2j_2} \bar{u}_{i'_1j'_1}\bar{u}_{i'_2j'_2}d\mu(U_H) \\
                =&\frac{\delta_{i_1i'_1} \delta_{i_2i'_2} \delta_{j_1j'_1} \delta_{j_2j'_2}+
                \delta_{i_1i'_2} \delta_{i_2i'_1} \delta_{j_1j'_2} \delta_{j_2j'_1}}{N^2-1} -\\
                &\frac{\delta_{i_1i'_1} \delta_{i_2i'_2} \delta_{j_1j'_2} \delta_{j_2j'_1}+
                \delta_{i_1i'_2} \delta_{i_2i'_1} \delta_{j_1j'_1} \delta_{j_2j'_2}}{N(N^2-1)},
            \end{split}
        \end{equation}
        we evaluate the integral appearing in the gradient variance. Subsequently, by performing trace expansions and applying Lemma \ref{lem:1} and \ref{lem:2} to simplify high-order terms, we obtain the bound $|\mathrm{Var}(\partial_{l,k} \mathcal{L}_H)| \leq \frac{8}{|\mathcal{S}|^4 \cdot (2^{2n_{data}}-1)}$.
\end{itemize}

\textbf{Approximate Haar measure (Eq. \ref{eq:L_ahaar_p})}: The analysis for the approximate Haar measure is almost identical to that for the exact Haar measure, except that a bias term $\varepsilon$ and $\zeta$ appear in the expectation and variance, respectively.

For a detailed proof, please refer to the Appendix \ref{ap:bp}.

\end{proofsketch}

As shown in Theorem~\ref{thm:bp}, it is clear that as the number of qubits increases, the variance of the gradient is upper bounded by the term that converges exponentially to zero as the number of qubits increases. Therefore, when the number of qubits increases, the gradient will converge exponentially to mean $0$. This indicates that when the number of qubits is large, QuDDPM indeed suffers from the BP phenomenon, which arises due to the use of Haar-random states as inputs to the PQC. Furthermore, as the earlier PQCs are not adequately trained, the subsequent PQCs will receive approximately Haar-random states as input. Consequently, the entire training process will be hindered by the BP phenomenon, rendering the model untrainable.

\begin{figure}[t]
    \vskip 0.2in
    \begin{center}
    \centerline{\includegraphics[width=0.7\columnwidth]{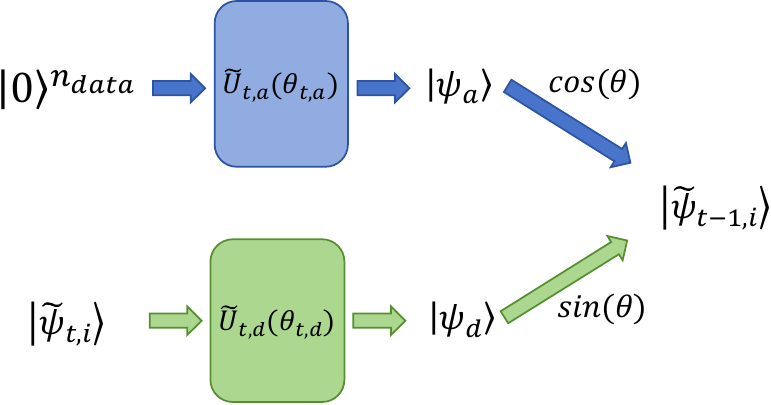}}
    \caption{The circuit diagram for a single training cycle within the backward process of our improved QuDDPM is illustrated, comprising two qubit systems and two independent PQCs, where $|\tilde{\psi}_{t-1,i}\rangle= \kappa(\cos(\theta)|\psi_a\rangle+\sin(\theta)|\psi_d\rangle)$, where $\kappa$ is the normalization factor.}
    \label{fig:improved_QuDDPM}
    \end{center}
    \vskip -0.2in
\end{figure}

\section{Improved and Conditional QuDDPM Architectures}

As mentioned above, the BP induced by Haar-random states renders the entire model difficult to train. This is because when Haar‑random states are used as input to a PQC, the presence of the BP forces the PQC’s output to remain, or closely approximate, a Haar‑random state, thereby creating a vicious cycle. To counteract this phenomenon, we propose an improved strategy: an additional auxiliary qubit system is introduced to perturb the output quantum state generated by the PQC acting on the data qubit system. These perturbations gradually steer the produced state away from the Haar distribution and toward the target quantum state.

\subsection{Improved QuDDPM} 
The core modification we have made to the original QuDDPM lies in the architecture of its backward process. As illustrated in Fig.\ref{fig:improved_QuDDPM}, in addition to the conventional data qubit system, we introduce an auxiliary qubit system initialized in the $|0\rangle$, which possesses the same number of qubits as the data qubit system. In the circuit design, two entirely independent PQCs are employed to process the auxiliary qubit system and the data qubit system separately. The output quantum states from the two PQCs are then entangled and superimposed to produce the final output state at the current timestep, expressed as 
\begin{equation}
    |\tilde{\psi}_{t-1,i}\rangle = \kappa[ \sin(\theta)\tilde{U}_{t,d}(\theta_{t,d})|\tilde{\psi}_{t,i}\rangle + \cos(\theta)\tilde{U}_{t,a}(\theta_{t,a})|0\rangle^{n_{data}}]
\end{equation}
where $\kappa$ is the normalization factor.

We note that when t is close to or equal to T, the states $|\tilde{\psi}_{t,i}\rangle$ approximately or exactly follow the Haar distribution. According to Theorem 3, introducing an auxiliary qubit system as in the original QuDDPM cannot alleviate the BP problem. In contrast, in our improved scheme, the auxiliary qubit system does not merely serve to add nonlinearity; rather, it acts as a guide and a symmetry breaker during training. That is, on one hand, it steers the generated states as much as possible toward the target state; on the other hand, the superposition of the quantum states from the auxiliary qubits with the data quantum states partially breaks the symmetry of the Haar-random states. The combination of these two effects holds promise for mitigating the BP, thereby partially restoring the trainability of the model. Regarding the specific circuit implementation of the improved QuDDPM, we propose a scheme that requires introducing only a single control qubit instead of $n_{data}$ auxiliary qubits. Detailed in Appendix \ref{ap:improved_circuit}, where a comprehensive description and explanation are provided, this scheme implements the improved QuDDPM through controlled operations, along with measurement and post-selection on the control qubits.

\subsection{Conditional QuDDPM} 
\subsubsection{Conditional state generation mechanism}
We have refined our model specifically for the generation of Hamiltonian ground states. Specifically, the parameters for the PQC and state entanglement are obtained via 
\begin{equation} \label{eq:conditional_theta}
    \theta = \tanh(W \cdot x + b),
\end{equation}
where $x$ denotes the Hamiltonian parameters, and $W$ and $b$ are trainable weights and biases, respectively. It is not difficult to see that our approach of mapping the Hamiltonian parameters $x$ to the circuit parameters $\theta$ is analogous to that of classical neural networks. By leveraging the expressive power of classical neural networks, it is sufficient to obtain the desired circuit parameters $\theta$ from the Hamiltonian parameters $x$. Through this scheme, we obtain a conditional QuDDPM that can be trained to generate the corresponding ground state of a Hamiltonian based on the input parameters.

\subsubsection{Distinctions from other conditional quantum denoising diffusion model}

We note that a concurrent work has also proposed a conditional generation model for QuDDPM, namely the Conditioned Quantum Denoising Diffusion (CQDD) model \cite{CQDD}. Although both CQDD and our proposed conditional QuDDPM are built upon the same base architecture of QuDDPM, the two models are completely different.

In terms of circuit architecture, the forward and backward processes of CQDD are identical to those of the original QuDDPM (i.e., Fig.~\ref{fig:circuit}). In contrast, the circuit structure of our conditional QuDDPM is based on the improved QuDDPM (i.e., Fig.~\ref{fig:improved_QuDDPM}), meaning that our conditional QuDDPM is capable of mitigating the BPs induced by Haar-random states, whereas CQDD inevitably suffers from the BPs.

Regarding the condition embedding mechanism, CQDD linearly maps $c$ distinct categories onto a set of uniformly spaced angles $\mu$ within the interval $0 \le \mu < 2\pi$, initializing the auxiliary qubits of the original QuDDPM via $(R_x(\mu)\vert{}0\rangle)^{\otimes n_a}$ to achieve condition embedding. Our method, conversely, maps conditional parameters to circuit parameters via a classical neural network. This allows our conditional QuDDPM to exert a more direct and profound influence on state generation, albeit at the expense of a larger training overhead. In contrast, CQDD operates in the opposite regime: its conditional embedding restricts the influence of conditions on the final outcomes, but introduces virtually no additional burden to the model training. A more profound distinction is that CQDD is limited to tasks with discrete conditional variables, whereas our conditional QuDDPM successfully accommodates continuous ones.

In summary, CQDD is well-suited for simple generative tasks with discrete conditional variables, and our conditional QuDDPM successfully extends to complex scenarios involving continuous conditional variables. This firmly demonstrates the distinct innovation of our proposed framework.

\begin{table*}[htbp]
\centering
\caption{\label{tab:experimental_settings}Experimental settings}
\vspace{0.2cm} 

\renewcommand{\arraystretch}{1.25} 
\begin{tabularx}{\textwidth}{L C C C C C}
\toprule
\textbf{Experiment} & \textbf{Data Qubits} & \textbf{PQC Layers} & \textbf{Training Cycles} & \textbf{Learning Rate} & \textbf{Optimizer} \\
\midrule
Experiment 1   & 10                   & 5                   & 40              & 0.01                   & AdamW              \\
Experiment 2   & 5                    & 5                   & 60              & 0.01                   & AdamW              \\
Experiment 3   & 5                    & 5                   & 40              & 0.01                   & AdamW              \\
\bottomrule
\end{tabularx}
\end{table*}

\section{Experiment}

In this section, we investigated the BP phenomenon and validated the effectiveness of conditional QuDDPM on two quantum systems: the one-dimensional antiferromagnetic Heisenberg model and the one-dimensional transverse-field Ising model.

\subsection{Experimental data}

\subsubsection{One-dimensional transverse-field Ising model} 
An $N$-qubit one-dimensional transverse-field Ising model is defined by the Hamiltonian
\begin{equation} \label{eq:tfim}
    H_{\text{tfim}} = -J \sum_{i=1}^{N-1} \sigma_z^{(i)} \sigma_z^{(i+1)} - h \sum_{i=1}^{N} \sigma_x^{(i)},
\end{equation}
where $J$ is the exchange coupling constant, $h$ is the magnetic field strength, and $\sigma_x^{(i)}$ ($\sigma_z^{(i+1)}$) denotes the Pauli‑X (Pauli‑Z) operator acting on the $i$-th qubit. Numerically, $J > 0$ and $h > 0$. According to the value of $J/h$, the ground state of the one-dimensional transverse-field Ising model corresponds to two distinct phases: it corresponds to a paramagnetic phase when $J/h < 1$ and a ferromagnetic phase when $J/h > 1$, with $J/h = 1$ being the phase transition point.

For the one‑dimensional transverse‑field Ising model, we collect two distinct datasets to serve different experimental purposes. The first dataset is used to investigate the BP problem in the original QuDDPM. We select the ground state at the phase transition point, i.e., $J/h = 1$, to ensure the complexity of the target state. The second dataset is employed to validate the effectiveness of the conditional QuDDPM. Therefore, we randomly sample multiple values of $J/h$ from the interval $[0.25, 4]$ to construct our training set.

\subsubsection{One‑dimensional antiferromagnetic Heisenberg model} 
An $N$-qubit one‑dimensional antiferromagnetic Heisenberg model is defined by the Hamiltonian
\begin{equation} \label{eq:heisenberg}
H_{\text{afhm}} = \sum_{i=1}^{N-1} J_i \left( \sigma_x^{(i)} \sigma_x^{(i+1)} + \sigma_y^{(i)} \sigma_y^{(i+1)} + \sigma_z^{(i)} \sigma_z^{(i+1)} \right).
\end{equation}
where $J_i$ denotes the coupling parameter, with $J_i > 0$ numerically.

For the one‑dimensional antiferromagnetic Heisenberg model, we primarily use it to validate the effectiveness of the conditional QuDDPM. Therefore, we take multiple sets of $J_i$ values from the interval $[0.5, 1.5]$ to form our training set.

\subsection{Hyper-parameters and Quantum state property}

\subsubsection{Hyper-parameters} 
As discussed above, our experiments aim to address two key objectives: (i) to empirically demonstrate the presence of BPs in the original QuDDPM and to validate that our improved QuDDPM effectively mitigates this issue; and (ii) to verify the efficacy of the conditional QuDDPM. To investigate the existence of BPs in the original QuDDPM and assess the performance of our improved variant, we employ a 10-qubit quantum system with a fixed PQC depth of 5 layers. For evaluating the conditional QuDDPM, which is designed to learn ground states corresponding to varying parameters of a given Hamiltonian, we adopt a smaller 5-qubit system. In this setting, the training dataset consists of 50 distinct ground states associated with the same Hamiltonian family. This choice strikes a balance between sample diversity and computational tractability, enabling efficient experimentation while preserving meaningful variation in the target states. The PQC depth is again fixed at 5 layers.

\subsubsection{Evaluation metrics}
To further validate the quality of our generated states, we introduce an important quantum state property, the two-point correlation function 
\begin{equation} \label{eq:cf}
    C_{ij} = \operatorname{tr}(O_{ij} \rho), 
\end{equation}
where $\rho$ is the density matrix of the quantum state, $i$ and $j$ denote two qubit sites in the system, and $O_{ij} = Z_i Z_j$.

\subsection{Experimental Settings}

As shown in Tab.~\ref{tab:experimental_settings}, the table presents the number of data qubits, the number of PQC layers used, the number of training cycles, the learning rate, and the optimizer corresponding to the three experiments. Experiment 1 corresponds to the comparison experiment between improved QuDDPM and original QuDDPM; Experiment 2 corresponds to the experiment where conditional QuDDPM learns and generates the One-dimensional antiferromagnetic Heisenberg model; Experiment 3 corresponds to the experiment where conditional QuDDPM learns and generates the One-dimensional transverse-field Ising model.

\subsection{Experimental Results}

Here, we provide numerical results for the experimental setting described earlier, demonstrating the existence of BPs in the original QuDDPM and validating the effectiveness of our proposed improved QuDDPM and conditional QuDDPM.

\begin{figure*}[t]
    \vskip 0.2in
    \begin{center}
    \centerline{\includegraphics[width=1.5\columnwidth]{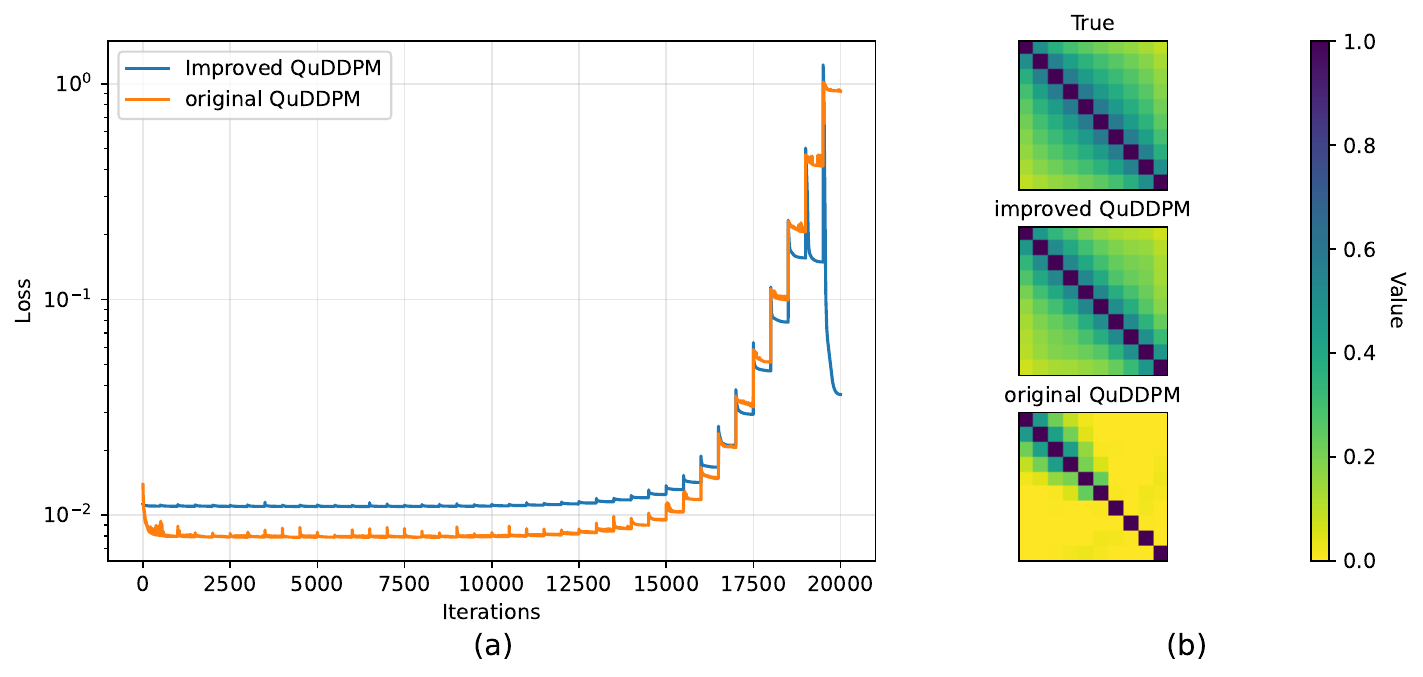}}
    \caption{Comparison of training performance between original and improved QuDDPM. The experiment is conducted to generate the ground state of a 10-qubit One-dimensional transverse-field Ising model at the phase transition point, where $J/h = 1$. Subfigure (a) shows the evolution of the loss function during training, while Subfigure (b) compares the two-point correlation functions of the true state with those generated by different models in the testing phase, where darker shades represent values closer to 1. Here, the horizontal and vertical axes correspond to the qubit sites.}
    \label{fig:exp1_loss_and_cf}
    \end{center}
    \vskip -0.2in
\end{figure*}

\begin{figure}[tb]
    \vskip 0.2in
    \begin{center}
    \centerline{\includegraphics[width=\columnwidth]{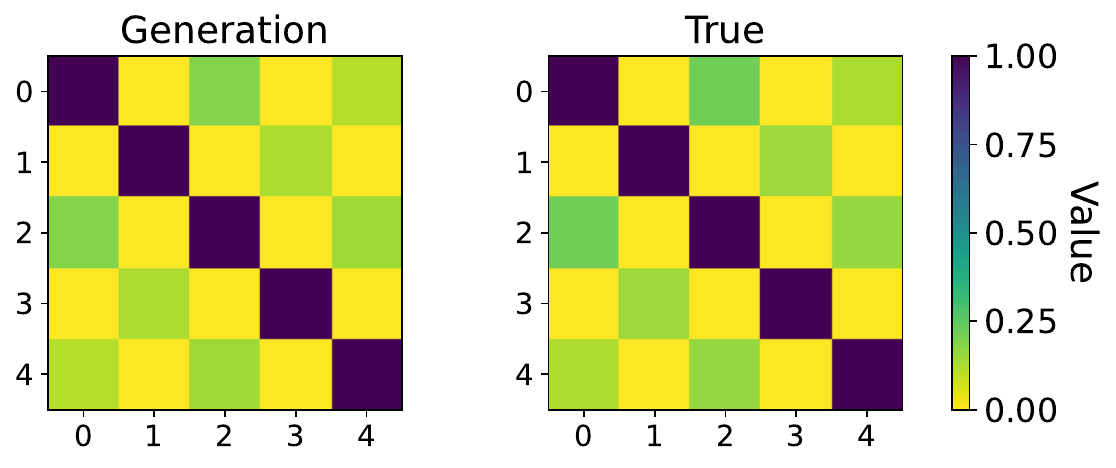}}
    \caption{Comparison between the true ground state and that generated by conditional QuDDPM for the one-dimensional antiferromagnetic Heisenberg model. The figure displays the two-point correlation functions of the generated and true ground states. Color intensity indicates the magnitude of the two-point correlation function.}
    \label{fig:afhm_cf}
    \end{center}
    \vskip -0.2in
\end{figure}

\subsubsection{Barren plateaus and mitigation} 
We first examine whether BPs exist in the original QuDDPM and whether the improved QuDDPM can alleviate BPs. Fig.~\ref{fig:exp1_loss_and_cf} presents a comparison between the original QuDDPM and the improved QuDDPM with respect to the loss function during the training process and the quality of the generated data in the experiment. Subfigure (a) presents the evolution of the loss function for both models during training. It can be observed that in the early stage of training, the loss functions of both our improved QuDDPM and the original QuDDPM show no significant changes, but the loss value of the original QuDDPM is slightly smaller. This is because, in the early stage of training, the input and output states in each training epoch are approximately Haar-distributed, resulting in generally minor fluctuations in the loss function. Yet, precisely due to the presence of the ancilla qubit system, our improved QuDDPM struggles to fully approximate the Haar distribution, leading to a slightly larger loss value in the early stage, but this very characteristic is what mitigates the BP problem. In the later stages of training, the original QuDDPM fails to effectively learn the target states due to the presence of BPs; consequently, accumulated errors increase progressively, and the generated states deviate increasingly from the target states. Meanwhile, although our improved QuDDPM also encounters approximate BP during the middle and late stages of training, these plateaus are rapidly mitigated, allowing the model to learn to generate states of substantially superior quality compared to those produced by the original QuDDPM. Subfigure (b) displays the two‑point correlation functions of the states generated by the two models during testing. From top to bottom, the first panel corresponds to the two‑point correlation function of the true state, the second panel shows the output of our improved QuDDPM, and the third panel presents the result from the original QuDDPM. Evidently, the original QuDDPM fails to achieve satisfactory generative performance, whereas our improved QuDDPM effectively learns and generates the corresponding state.

\begin{figure}[t]
    \vskip 0.2in
    \begin{center}
    \centerline{\includegraphics[width=1\columnwidth]{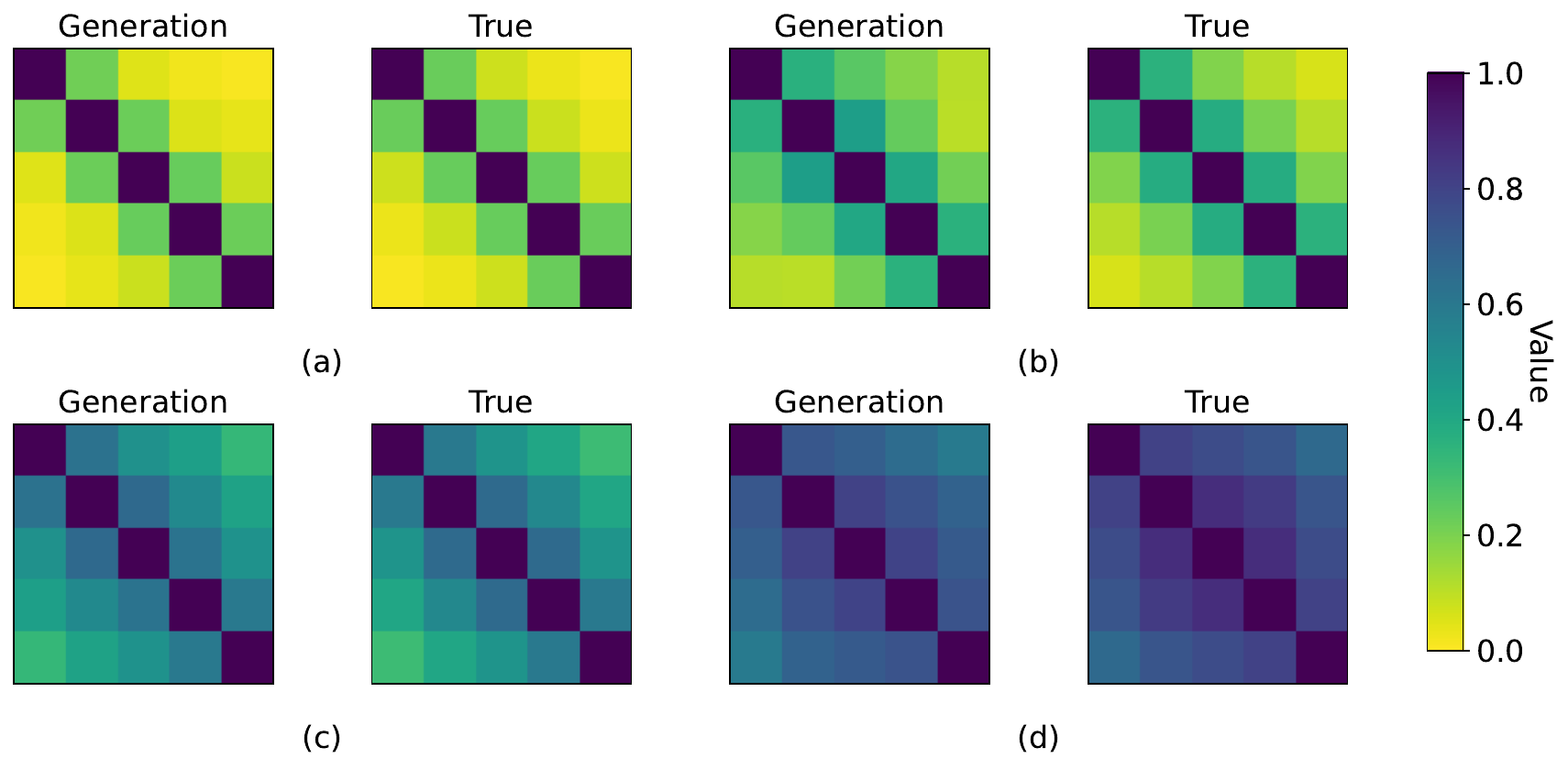}}
    \caption{Comparison between the true ground state and that generated by the conditional QuDDPM for one-dimensional transverse-field Ising model. (a), (b), (c), and (d) correspond to the comparisons of the two-point correlation functions between the generated quantum states (left) and the target quantum states (right) for $J/h \in (0, 0.5)$, $(0.5, 1)$, $(1, 1.5)$, and $(1.5, 2)$, respectively.}
    \label{fig:tfim_cf}
    \end{center}
    \vskip -0.2in
\end{figure}
\subsubsection{Performance of conditional QuDDPM} 
Fig.~\ref{fig:afhm_cf} presents the experimental results of the conditional QuDDPM learning and generating the ground state of the one‑dimensional antiferromagnetic Heisenberg model, where the left panel displays the two‑point correlation function of the state generated by the conditional QuDDPM and the right panel shows the corresponding ground‑truth values. The fidelity between the two states is 0.9168. The experimental results clearly indicate that the states produced by the conditional QuDDPM can capture the characteristics of the true state, providing preliminary validation of the effectiveness of our conditional QuDDPM. 

We further validate the effectiveness of our conditional QuDDPM by learning the ground states of the one‑dimensional transverse‑field Ising model. As shown in Fig.~\ref{fig:tfim_cf}, the two‑point correlation functions of the quantum states generated by the conditional QuDDPM in the four regions are almost identical to the truth values, indicating that our conditional QuDDPM can effectively generate the corresponding quantum states based on the input conditions. To further experimentally validate the effectiveness of our model, we present the corresponding fidelities for the four subfigures (a), (b), (c), and (d) in Fig.~\ref{fig:tfim_cf}, which are 0.9954, 0.9967, 0.9672, and 0.9717, respectively. As can be seen, the generated states are extremely close to the exact states. Fig.~\ref{fig:tfim_sc} displays the distribution of 5,000 generated states projected onto a two‑dimensional plane via t-distributed stochastic neighbor embedding (t-SNE) \cite{t-sne}. It is clearly visible that the states corresponding to the ferromagnetic phase (blue dots) and the paramagnetic phase (orange dots) form two distinct clusters, separated by a clear empty region.
This further demonstrates that our generated states have successfully captured the essential properties of the quantum states associated with the Hamiltonian.

\begin{figure}[t]
    \vskip 0.2in
    \begin{center}
    \centerline{\includegraphics[width=0.73\columnwidth]{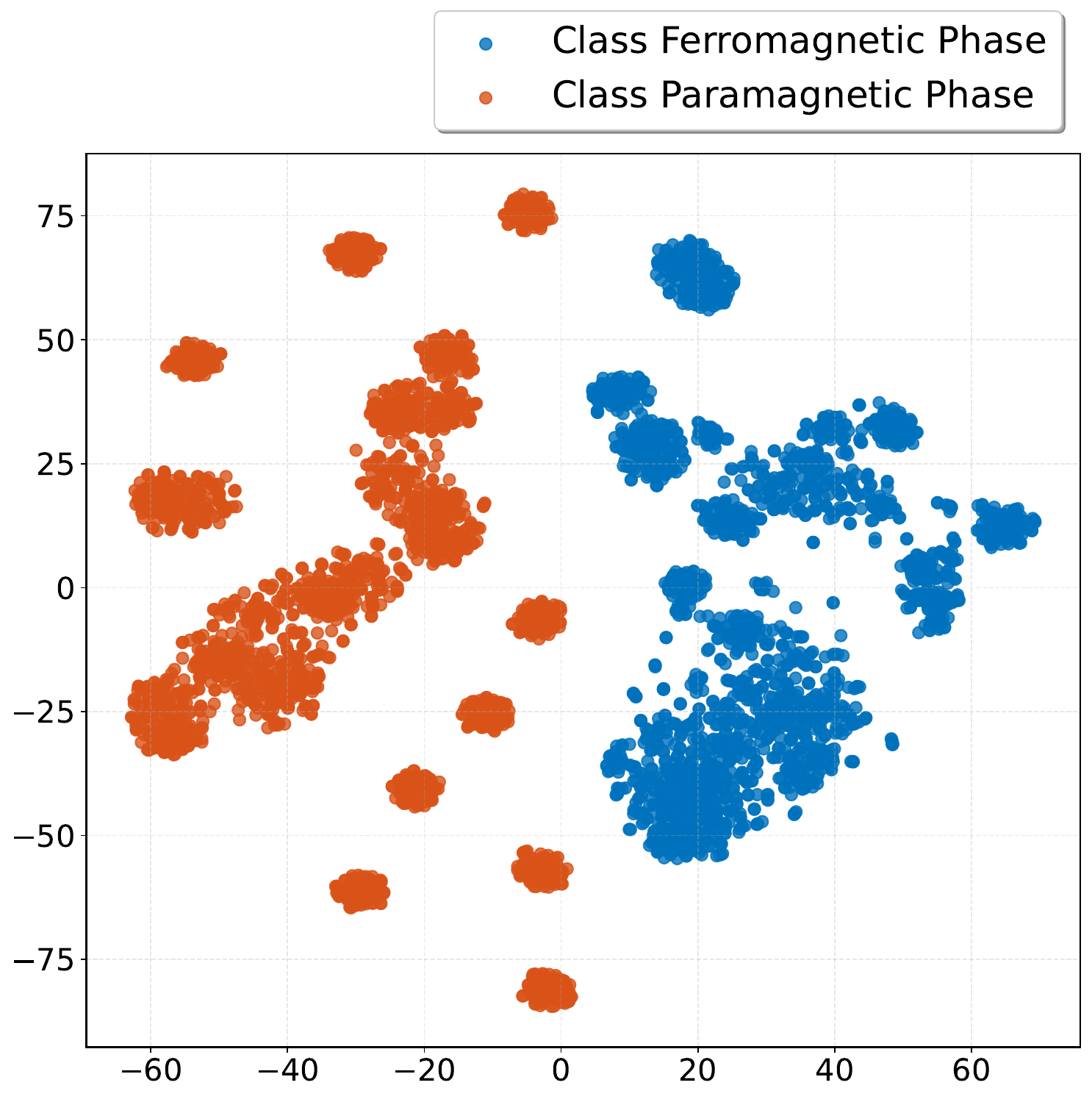}}
    \caption{t-SNE visualization of ground states generated by conditional QuDDPM for the one-dimensional transverse-field Ising model. States are mapped onto a 2D plane and labeled according to their quantum phases, computed from the corresponding Hamiltonian parameters. Distinct clusters correspond to different phases of matter.}
    \label{fig:tfim_sc}
    \end{center}
    \vskip -0.2in
\end{figure}

\section{Conclusion and Discussion}

In this work, we first identify the existence of BPs in QuDDPM and pinpoint the underlying cause: the initial input to the backward denoising process of QuDDPM consists of a set of Haar-random states. We further elaborate on the severe consequences induced by this phenomenon; specifically, due to the BP, the output of Haar-random states after passing through a PQC remains Haar-random or approximately Haar-random, trapping QuDDPM in the BP throughout the entire training process. To address this, we improve upon the original QuDDPM to propose a novel quantum diffusion probabilistic model. In the tested settings, the BP in our improved QuDDPM is effectively mitigated. Regarding the circuit implementation of the improved QuDDPM, we present a scheme that introduces only a single additional control qubit. Instead of adding $n_{data}$ auxiliary qubits, this scheme utilizes a controlled-reset circuit $U_{\text{reset}}$, which may contain $\Omega(n)$ layers. Although $U_{\text{reset}}$ does not include any trainable parameters, its presence reduces the noise tolerance of the improved QuDDPM, thereby affecting the final success rate. In terms of measurement overhead, our proposed circuit scheme requires only a single measurement on a single control qubit at the end of the circuit, along with one step of post-selection processing. As shown in Appendix \ref{ap:improved_circuit}, the post-selection has a success probability of 50\%, so the measurement overhead is entirely tolerable. Naturally, we also hope to find circuit schemes with higher noise tolerance and lower resource overhead, for instance, by leveraging machine learning methods to find a $U_{\text{reset}}$ with only $\Omega(1)$ layers, which represents another interesting topic for future research. Moreover, we further refine our model into a conditional QuDDPM, enabling it to generate the corresponding ground state based on the input Hamiltonian parameters. This expands the application scope of quantum diffusion models and advances the development of quantum diffusion models. Finally, it is regrettable that, due to equipment limitations, we were unable to conduct experiments on a larger scale, which may render our experimental results somewhat insufficient. Nevertheless, at least it can be concluded that, in representative small-scale systems, our proposed improvement shows promising potential in mitigating the BP caused by Haar-random states.

\begin{acknowledgments}
    The authors were partially supported by the National Natural Science Foundation of China (Grant No. 62002333), the Shenzhen International Quantum Academy (Grant No. SIQA2024KFKT01) and the Innovation Program for Quantum Science and Technology (Grant No. 2021ZD0302901).
\end{acknowledgments}

\bibliographystyle{unsrt}
\bibliography{references}

\clearpage
\onecolumngrid
\appendix

\section{Partial derivatives of PQCs}
\label{ap:U_gradient}
For any $\tilde{U}_t({\bm \theta_t})$ that conforms to 
\begin{equation}  \label{eq:U_t}
    \tilde{U}_t(\bm \theta_t)=\prod\limits_{l=1}^{L}W_t V_t(\bm \theta_{t,l}),
\end{equation}
we have
\begin{equation}
    \begin{split}
        \partial_{l,k}\tilde{U}_t({\bm \theta_t})=&\tilde{U}_{t,L:l+1}W_t \cdot {} [\otimes_{\lambda=0}^{n-1} \prod\limits_{i=0}^{\alpha}R_{\sigma^{(i)}}(\theta_{t,l,\lambda\tau+i})\cdot(-\frac{i}{2}\sigma^{(\alpha)} \delta_{\lambda\tau+\alpha,k}+I \cdot(1- \delta_{\lambda\tau+\alpha,k})) \\
        &\cdot\prod\limits_{i=\alpha+1}^{\tau-1}R_{\sigma^{(i)}}(\theta_{t,l,\lambda\tau+i})]\cdot \tilde{U}_{t,l-1:1} \\
        =&-\frac{i}{2}\tilde{U}_{t,L:l}K\tilde{U}_{t,l-1:1},
    \end{split}
\end{equation}
where $\alpha=k\mod \tau$, n is the total number of qubits, $K=(\otimes_{\lambda=0}^{n-1}\prod\limits_{i=\alpha+1}^{\tau-1}R_{\sigma^{(i)}}(\theta_{t,l,\lambda\tau+i}))^\dagger[(I^{0:\lfloor k/\tau\rfloor}\\\otimes \sigma^{(\alpha)} \otimes I^{\lfloor k/\tau+1\rfloor:n})(\otimes_{\lambda=0}^{n-1}\prod\limits_{i=\alpha+1}^{\tau-1}R_{\sigma^{(i)}}(\theta_{t,l,\lambda\tau+i}))$, $I^{i:j}$ denotes that the operator $I$ acts on the qubits in the range $[i,j)$. Obviously $K$ is a Hermitian operator.

\section{Partial derivatives of loss}
\label{ap:loss_gradient}
The loss $\mathcal{L}_t(\bm \theta_t)$ that conforms to $\mathcal{L}_t(\bm \theta_t) =\mathcal{D}_{\rm MMD}(\tilde{\mathcal{S}}_{t-1}, \mathcal{S}^{'}_{t-1}) =\mathcal{D}_{\rm MMD}(\tilde{U}_t(\bm \theta_t)\tilde{\mathcal{S}}_t, \mathcal{S}^{'}_{t-1})$ can be regarded as the function of $\tilde{U}_t(\bm \theta_t)$, therefore, we have:

\begin{itemize}

\item[$\bullet$] $\bar{F}(\mathcal{S}_{t-1}^{'},\mathcal{S}_{t-1}^{'})$ is a constant term, so $\partial_{l,k}\bar{F}(\mathcal{S}_{t-1}^{'},\mathcal{S}_{t-1}^{'})=0$

\item[$\bullet$] For $\bar{F}(\tilde{S}_{t-1},\tilde{S}_{t-1})$, we have
\begin{equation}
    \begin{split}
        \bar{F}(\tilde{S}_{t-1},\tilde{S}_{t-1})
        =&\mathbb{E}_{|\psi_{t-1}\rangle\in\tilde{S}_{t-1},|\phi_{t-1}\rangle\in\tilde{S}_{t-1}}[|\langle\psi_{t-1}|\phi_{t-1}\rangle|^2]\\
        =&\mathbb{E}_{\rho_{\psi_{t-1}}\in\tilde{S}_{t-1},\rho_{\phi_{t-1}}\in\tilde{S}_{t-1}}{\rm Tr}(\rho_{\psi_{t-1}}\cdot\rho_{\phi_{t-1}}),
    \end{split}
\end{equation} 
where $\rho_{\psi_{t-1}}$ and $\rho_{\psi_{t-1}}$ are the density matrices of $\psi_{t-1}$ and $\phi_{t-1}$, respectively. Then
\begin{equation}
    \begin{split}
        \rho_{\psi_{t-1}}
        =&{\rm Tr_{ancilla}}[\tilde{U}_t(|0\rangle\langle0|\otimes\rho_{\psi_{t}})]\\
        =&{\rm Tr}(\tilde{U}_{t}^{0:n_a}|0\rangle\langle0|\tilde{U}_t^{0:n_a\dagger})\cdot\tilde{U}_{t}^{n_a+1:n}\rho_{\psi_t}\tilde{U}_{t}^{n_a+1:n\dagger}\\
        =&\tilde{U}_{t}^{n_a+1:n}\rho_{\psi_t}\tilde{U}_{t}^{n_a+1:n\dagger},
    \end{split}
\end{equation}
where ${\rm Tr_{ancilla}}$ denotes the partial trace over the ancilla qubits to obtain the reduced density operator for the data qubits. $\tilde{U}_{t}^{0:n_a}$ and $\tilde{U}_{t}^{n_a+1:n}$ represents the part of $\tilde{U}_{t}$ applied to the ancilla qubits and data qubits, respectively. Then, we have
\begin{equation}
    \begin{split}
        \bar{F}(\tilde{S}_{t-1},\tilde{S}_{t-1})
        =&\mathbb{E}_{\rho_{\psi_{t-1}}\in\tilde{S}_{t-1},\rho_{\phi_{t-1}}\in\tilde{S}_{t-1}}{\rm Tr}(\rho_{\psi_{t-1}}\cdot\rho_{\phi_{t-1}})\\
        =&\mathbb{E}_{\rho_{\psi_{t}}\in\tilde{S}_{t},\rho_{\phi_{t}}\in\tilde{S}_{t}}{\rm Tr}(\tilde{U}_{t}^{n_a+1:n}\rho_{\psi_t}\tilde{U}_{t}^{n_a+1:n\dagger}\cdot\tilde{U}_{t}^{n_a+1:n}\rho_{\phi_t}\tilde{U}_{t}^{n_a+1:n\dagger})\\
        =&\mathbb{E}_{\rho_{\psi_{t}}\in\tilde{S}_{t},\rho_{\phi_{t}}\in\tilde{S}_{t}}{\rm Tr}(\rho_{\psi_{t}}\cdot\rho_{\phi_{t}})\\
        =&\bar{F}(\tilde{S}_{t},\tilde{S}_{t})\\
        =&\bar{F}(\tilde{S}_{T},\tilde{S}_{T}).
    \end{split}
\end{equation}
Therefore, $\bar{F}(\tilde{S}_{t-1},\tilde{S}_{t-1})$ is also a constant term, so $\partial_{l,k}\bar{F}(\tilde{S}_{t-1},\tilde{S}_{t-1})=0$.

\item[$\bullet$] For $\bar{F}(\tilde{S}_{t-1},S_{t-1}^{'})$, we have
\begin{equation}
    \begin{split}
        \partial_{l,k}\bar{F}(\tilde{S}_{t-1},S_{t-1}^{'})
        =&\mathbb{E}_{|\tilde{\psi}_{t-1}\rangle\in\tilde{S}_{t-1},|\psi_{t-1}\rangle\in S_{t-1}^{'}}\partial_{l,k}[|\langle\tilde{\psi}_{t-1}|\psi_{t-1}\rangle|^2] \\
        =&\mathbb{E}_{|\tilde{\psi}_{t}\rangle\in\tilde{S}_{t},|\psi_{t-1}\rangle\in S_{t-1}^{'}}\partial_{l,k}[|\langle\tilde{\psi}_{t}|\tilde{U}_t^{'}|\psi_{t-1}\rangle|^2] \\
        =&\mathbb{E}_{\rho_{\tilde{\psi}_{t}}\in\tilde{S}_{t},\rho_{\psi_{t-1}}\in S_{t-1}^{'}}\partial_{l,k}{\rm Tr}[\rho_{\psi_{t-1}}\cdot\tilde{U}_{t}^{'}\rho_{\tilde{\psi}_t}\tilde{U}_{t}^{'\dagger}] \\
        =&\mathbb{E}_{\rho_{\tilde{\psi}_{t}}\in\tilde{S}_{t},\rho_{\psi_{t-1}}\in S_{t-1}^{'}}{\rm Tr}[\rho_{\psi_{t-1}}\cdot\partial_{l,k}\tilde{U}_{t}^{'}\rho_{\tilde{\psi}_t}\tilde{U}_{t}^{'\dagger}] \\
        =&\frac{1}{|\tilde{S}_t||S_{t-1}^{'}|}\cdot(-\frac{i}{2})\sum_{i,j}{\rm Tr}(\rho_{\psi_{t-1,j}}\cdot\tilde{U}_{t,L:l}^{'}[K,\tilde{U}_{t,l-1:1}^{'}\rho_{\tilde{\psi}_{t,i}}\tilde{U}_{t,l-1:1}^{'\dagger}]\tilde{U}_{t,L:l}^{'\dagger}),
    \end{split}
\end{equation}
where $|\psi_{t-1,j}\rangle\in\mathcal{S}_{t-1}$, $|\tilde{\psi}_{t,i}\rangle\in\tilde{\mathcal{S}}_t$, $\tilde{U}^{'}_{t}$ represents the part of $\tilde{U}_{t}$ applied to the data qubits.

\end{itemize}

Thus, the partial derivative of the loss function $\mathcal{L}_t({\bm \theta_t})$ with respect to the parameters $\theta_{t,l,k}$ is:
\begin{equation}
    \begin{split}
        \partial_{l,k}\mathcal{L}_t({\bm \theta_t})
        =&-2\partial_{l,k}\bar{F}(\tilde{S}_{t-1},S_{t-1}^{'}) \\
        =&\frac{i}{|\mathcal{S}|^2}\sum\limits_{i,j}{\rm Tr}(|\psi_{t-1,j}\rangle\langle\psi_{t-1,j}|\cdot\tilde{U}_{t,L:l}^{'}[K,\tilde{U}_{t,l-1:1}^{'}|\tilde{\psi}_{t,i}\rangle\langle\tilde{\psi}_{t,i}|\tilde{U}_{t,l-1:1}^{'\dagger}]\tilde{U}_{t,L:l}^{'\dagger}),
    \end{split}
\end{equation}
where $|\mathcal{S}|=|\mathcal{S}_{t-1}|=|\tilde{\mathcal{S}}_{t-1}|$.

\section{Proof of Lemma \ref{lem:1}}
\label{ap:lem1}

For ${\rm Tr}(AB[V,CU|0\rangle \langle0|U^\dagger C^\dagger]B^\dagger)$ we have
\begin{equation}
    \begin{split}
        &{\rm Tr}(AB[V,CU|0\rangle \langle0|U^\dagger C^\dagger]B^\dagger)\\
        =&{\rm Tr}(ABVCU|0\rangle \langle0|U^\dagger C^\dagger B^\dagger -
        ABCU|0\rangle \langle0|U^\dagger C^\dagger V B^\dagger)\\
        =&\sum (ABVC)_{ib}U_{bc}(|0\rangle \langle0|)_{cd}U^\dagger_{de}(C^\dagger
        B^\dagger)_{ei}-(ABC)_{jb'}U_{b'c'}(|0\rangle \langle0|)_{c'd'}U^\dagger_{d'e'}(C^\dagger 
        V B^\dagger)_{e'j} .
    \end{split}
\end{equation}

Letting $D=ABVC$, $E=C^\dagger B^\dagger$, $F=ABC$, $G=C^\dagger V B^\dagger$, we obtain the following result:
\begin{equation}
    \begin{split}
        &{\rm Tr}(AB[V,CU|0\rangle \langle0|U^\dagger C^\dagger]B^\dagger)\\
        =&\sum D_{ib}U_{bc}|0\rangle \langle0|_{cd}U^\dagger_{de}E_{ei}-F_{jb'}U_{b'1}|0\rangle \langle0|
        U^\dagger_{1e'}G_{e'j}\\
        =&\sum\limits_{\substack{b,b',e,e' \\<2^{n_{data}}}}D_{ib}U_{b1}U^\dagger_{1e}E_{ei}- F_{jb'}U_{b'c'}|0\rangle \langle0|_{c'd'}U^\dagger_{d'e'}G_{e'j}.
    \end{split}
\end{equation}

\section{Proof of Lemma \ref{lem:2}}
\label{ap:lem2}

For $\mathbb{E}_{U \in \mathcal{U}(N)} (\sum A_{ae}U_{e1}U_{1f}^\dagger B_{fa}C_{bg}U_{g1}U_{1h}^\dagger D_{hb})s$, we have
\begin{equation}
    \begin{split}
        \mathbb{E}_{U \in \mathcal{U}(N)} (\sum A_{ae}U_{e1}U_{1f}^\dagger B_{fa}C_{bg}U_{g1}U_{1h}^\dagger D_{hb})
        =\int d\mu(U) A_{ae}U_{e1}U_{1f}^\dagger B_{fa}C_{bg}U_{g1}U_{1h}^\dagger D_{hb}.
        \end{split}
\end{equation}

By applying Eq.\ref{eq:unitary_4} and assuming that N is sufficiently large, we obtain the following result:
\begin{equation}
    \begin{split}
        &\mathbb{E}(\sum A_{ae}U_{e1}U_{1f}^\dagger B_{fa}C_{bg}U_{g1}U_{1h}^\dagger D_{hb}) \\
        \approx & \frac{2}{N^2-1}\sum A_{a1}B_{1a}C_{b1}D_{1b}\\
        = & \frac{2}{N^2-1}(BA)_{11}(DC)_{11},
    \end{split}
\end{equation}

Due to A, B, C, and D are unitary matrices, $\mathbb{E}(\sum A_{ae} U_{e1}U_{1f}^\dagger B_{fa}C_{bg}U_{g1} U_{1h}^\dagger D_{hb}) \in [-\frac{2}{N^2-1}, \frac{2}{N^2-1}]$.

\section{Proof of Theorem \ref{thm:bp}}
\label{ap:bp}
Here we give the proof of Theorem.\ref{thm:bp}. For the gradients Eq.\ref{eq:L_haar_p}, we provide its mean over the unitary group $\mathcal{U}_H(N)$ which obey Haar measure \cite{theory_unitart_group}:
\begin{equation} \label{eq:gradient_mean}
    \langle  \partial_{l,k}\mathcal{L}_H\rangle
    =\int d\mu(U_H)\frac{i}{|\mathcal{S}|^2}\sum\limits_{i,j}{\rm Tr}(|\psi_{j}\rangle\langle\psi_{j}|\cdot\tilde{U}_{L:l}^{'}[K,\tilde{U}_{l-1:1}^{'}\cdot U_{Hi}|0\rangle\langle0|U_{Hi}^\dagger \cdot\tilde{U}_{l-1:1}^{'\dagger}]\tilde{U}_{L:l}^{'\dagger}),
\end{equation}

From Eq.~\ref{eq:unitary_2}, we can deduce that
\begin{equation} \label{eq:unitary_2_extension}
    \int_{\mathcal{U}_H(N)} d\mu(U_H)U_HVU_H^\dagger=\frac{{\rm Tr}(V)}{N}I.
\end{equation}
Then, we can calculate:
\begin{equation}
    \begin{split}
        \langle\partial_{l,k}\mathcal{L}_H\rangle
        =&\frac{i}{|\mathcal{S}|^2}\sum\limits_{i,j}{\rm Tr}(|\psi_{j}\rangle\langle\psi_{j}|\cdot\tilde{U}_{L:l}^{'}[K,\frac{I}{2^{n_{data}}}]\tilde{U}_{L:l}^{'\dagger}) \\
        =&0,
    \end{split}
\end{equation}
where $n_{data}$ represent the number of data qubits.

Next, we compute the variance of the gradient Eq.\ref{eq:L_haar_p}, which is given as
\begin{equation}
    {\rm Var}(\partial_{l,k}\mathcal{L}_H)=\langle(\partial_{l,k}\mathcal{L}_H)^2\rangle-\langle\partial_{l,k}\mathcal{L}_H\rangle^2.
\end{equation}
Due to $\langle\partial_{l,k}\mathcal{L}_H\rangle=0$, we can get=
\begin{equation}
    {\rm Var}(\partial_{l,k}\mathcal{L}_H)=\langle(\partial_{l,k}\mathcal{L}_H)^2\rangle.
\end{equation}

Expand as follows:
\begin{equation}
    \begin{split}
        {\rm Var}(\partial_{l,k}\mathcal{L}_H)
        =-\frac{1}{|\mathcal{S|}^4}(\sum\limits_{b\neq d}[\int d\mu(U_{Hb})&{\rm Tr}(\rho_{\psi_a}\cdot\tilde{U}_{L:l}^{'}[K,\tilde{U}_{l-1:1}^{'}U_{Hb}|0\rangle\langle0|U_{Hb}^\dagger\tilde{U}_{l-1:1}^{'\dagger}]\tilde{U}_{L:l}^{'\dagger}) \\ 
        \int d\mu(U_{Hd})&{\rm Tr}(\rho_{\psi_c}\cdot\tilde{U}_{L:l}^{'}[K,\tilde{U}_{l-1:1}^{'}U_{Hd}|0\rangle\langle0|U_{Hd}^\dagger\tilde{U}_{l-1:1}^{'\dagger}]\tilde{U}_{L:l}^{'\dagger})] \\
        +\sum[\int d\mu(U_{Hx})&{\rm Tr}(\rho_{\psi_i}\cdot\tilde{U}_{L:l}^{'}[K,\tilde{U}_{l-1:1}^{'}U_{Hx}|0\rangle\langle0|U_{Hx}^\dagger\tilde{U}_{l-1:1}^{'\dagger}]\tilde{U}_{L:l}^{'\dagger}) \\
        &{\rm Tr}(\rho_{\psi_j}\cdot\tilde{U}_{L:l}^{'}[K,\tilde{U}_{l-1:1}^{'}U_{Hx}|0\rangle\langle0|U_{Hx}^\dagger\tilde{U}_{l-1:1}^{'\dagger}]\tilde{U}_{L:l}^{'\dagger})]),
    \end{split}
\end{equation}
where $\rho_{\psi_k}=|\psi_k\rangle\langle\psi_k|$. After that, we can calculate:
\begin{equation} \label{eq:gradient_var}
    \begin{split}
        {\rm Var}(\partial_{l,k}\mathcal{L}_H)
        =&\langle\partial_{l,k}\mathcal{L}_H\rangle^2
        -\frac{1}{|\mathcal{S|}^4}\sum[\int d\mu(U_{Hx}){\rm Tr}(\rho_{\psi_i}\cdot\tilde{U}_{L:l}^{'}[K,\tilde{U}_{l-1:1}^{'}U_{Hx}|0\rangle\langle0|U_{Hx}^\dagger\tilde{U}_{l-1:1}^{'\dagger}]\\&\tilde{U}_{L:l}^{'\dagger}) 
        {\rm Tr}(\rho_{\psi_j}\cdot\tilde{U}_{L:l}^{'}[K,\tilde{U}_{l-1:1}^{'}U_{Hx}|0\rangle\langle0|U_{Hx}^\dagger\tilde{U}_{l-1:1}^{'\dagger}]\tilde{U}_{L:l}^{'\dagger})]\\
        =&-\frac{1}{|\mathcal{S|}^4}\sum[\int d\mu(U_{Hx}){\rm Tr}(\rho_{\psi_i}\cdot\tilde{U}_{L:l}^{'}[K,\tilde{U}_{l-1:1}^{'}U_{Hx}|0\rangle\langle0|U_{Hx}^\dagger\tilde{U}_{l-1:1}^{'\dagger}]\tilde{U}_{L:l}^{'\dagger}) \\&
        {\rm Tr}(\rho_{\psi_j}\cdot\tilde{U}_{L:l}^{'}[K,\tilde{U}_{l-1:1}^{'}U_{Hx}|0\rangle\langle0|U_{Hx}^\dagger\tilde{U}_{l-1:1}^{'\dagger}]\tilde{U}_{L:l}^{'\dagger})],
    \end{split}
\end{equation}

We similarly use Eq.~\ref{eq:unitary_4} to deal with the integral in Eq.\ref{eq:gradient_var}. Meanwhile, in accordance with Lemma \ref{lem:1}, we define $A_k=\rho_{\psi_k}\tilde{U}_{L:l}^{'}K\tilde{U}_{l-1:1}^{'}$, $B=\tilde{U}_{l-1:1}^{'\dagger}\tilde{U}_{L:l}^{'\dagger}=\tilde{U}^{'\dagger}$, $C_k=\rho_{\psi_k}\tilde{U}_{L:l}^{'}\tilde{U}_{l-1:1}^{'}=\rho_{\psi_k}\tilde{U}^{'}$, $D=\tilde{U}_{l-1:1}^{'\dagger}K\tilde{U}_{L:l}^{'\dagger}$ in Eq.\ref{eq:gradient_var}. Then, we obtain the following result:
\begin{equation}
    \begin{split}
        {\rm Var}(\partial_{l,k}\mathcal{L}_H)
        =&-\frac{1}{|\mathcal{S}|^4}\sum\limits_{i,j,x}\{\int d\mu(U_{Hx})
        [\sum(A_i)_{ab}(U_{Hx})_{b1}(U^\dagger_{Hx})_{1c}B_{ca}-
        (C_i)_{de}(U_{Hx})_{e1}\\&(U^\dagger_{Hx})_{1f}D_{fd}]
        \cdot [\sum(A_j)_{a'b'}(U_{Hx})_{b'1}(U^\dagger_{Hx})_{1c'}B_{c'a'}-
        (C_j)_{d'e'}(U_{Hx})_{e'1}\\&(U^\dagger_{Hx})_{1f'}D_{f'd'}]
        \}\\
        =&\frac{1}{|\mathcal{S}|^4}\sum\limits_{i,j,x}\{\int d\mu(U_{Hx})
        [-\sum(A_i)_{ab}(U_{Hx})_{b1}(U^\dagger_{Hx})_{1c}B_{ca}(A_j)_{a'b'}(U_{Hx})_{b'1}\\& (U^\dagger_{Hx})_{1c'}B_{c'a'}+
        \sum(A_i)_{ab}(U_{Hx})_{b1}(U^\dagger_{Hx})_{1c}B_{ca}(C_j)_{d'e'}(U_{Hx})_{e'1}(U^\dagger_{Hx})_{1f'}\\&D_{f'd'}+
        \sum(C_i)_{de}(U_{Hx})_{e1}(U^\dagger_{Hx})_{1f}D_{fd}(A_j)_{a'b'}(U_{Hx})_{b'1}(U^\dagger_{Hx})_{1c'}B_{c'a'}-\\&
        \sum(C_i)_{de}(U_{Hx})_{e1}(U^\dagger_{Hx})_{1f}D_{fd}(C_j)_{d'e'}(U_{Hx})_{e'1}(U^\dagger_{Hx})_{1f'}D_{f'd'}
        ]\}
    \end{split}
\end{equation}

By applying Lemma\ref{lem:2}, we obtain the following result:
\begin{equation}
    |{\rm Var}(\partial_{l,k} \mathcal{L}_H)| \leq \frac{1}{|\mathcal{S}|^4}\cdot\frac{2}{2^{2n_{data}}-1}\cdot4
    =\frac{8}{|\mathcal{S}|^4(2^{2n_{data}}-1)}.
\end{equation}

Next, we conduct an analysis of approximate Haar-random states. For approximate Haar-random states, we present the following properties:
\begin{equation}
    \int_{approximate\,Haar}U_{A}\rho U_{A}^{\dagger}d(U_A)=\int_{Haar}U_H\rho U_H^\dagger d(U_H)+\varepsilon,
\end{equation}
\begin{equation}
    \begin{split}
        &\int_{approximate\,Haar}(U_A)_{i_1j_1}(U_A)_{i_2j_2}(U_A^\dagger)_{i_1'j_1'}(U_A^\dagger)_{i_2'j_2'}d(U_A)\\
        =&\int_{Haar}(U_H)_{i_1j_1}(U_H)_{i_2j_2}(U_H^\dagger)_{i_1'j_1'}(U_H^\dagger)_{i_2'j_2'} d(U_H)+\zeta,
    \end{split}
\end{equation}
where $\varepsilon, \zeta$ denote the deviations of approximate Haar-random states from exact Haar-random states.

Therefore, for Eq.\ref{eq:L_ahaar_p}, it is easy to obtain:
\begin{equation}
    \langle \partial_{l,k}\mathcal{L}_A\rangle=\varepsilon,
\end{equation}
\begin{equation}
    \begin{split}
        {\rm Var}(\partial_{l,k}\mathcal{L}_A)=&\langle (\partial_{l,k}\mathcal{L}_A)^2\rangle-\langle \partial_{l,k}\mathcal{L}_A\rangle^2 \\
        =&\langle\partial_{l,k}\mathcal{L}_A\rangle^2
        -\frac{1}{|\mathcal{S|}^4}\sum[\int d\mu(U_{Ax}){\rm Tr}(\rho_{\psi_i}\cdot\tilde{U}_{L:l}^{'}[K,\tilde{U}_{l-1:1}^{'}U_{Ax}|0\rangle\langle0|U_{Ax}^\dagger\tilde{U}_{l-1:1}^{'\dagger}]\\&\tilde{U}_{L:l}^{'\dagger}) 
        {\rm Tr}(\rho_{\psi_j}\cdot\tilde{U}_{L:l}^{'}[K,\tilde{U}_{l-1:1}^{'}U_{Ax}|0\rangle\langle0|U_{Ax}^\dagger\tilde{U}_{l-1:1}^{'\dagger}]\tilde{U}_{L:l}^{'\dagger})]-\langle \partial_{l,k}\mathcal{L}_A\rangle^2 \\
        \leq&\frac{8}{|\mathcal{S}|^4(2^{2n_{data}}-1)}+4\zeta
    \end{split}
\end{equation}

\section{Circuit construction for the state superposition in improved QuDDPM}
\label{ap:improved_circuit}

\begin{figure*}[t]
    \vskip 0.2in
    \begin{center}
    \centerline{\includegraphics[width=0.7\columnwidth]{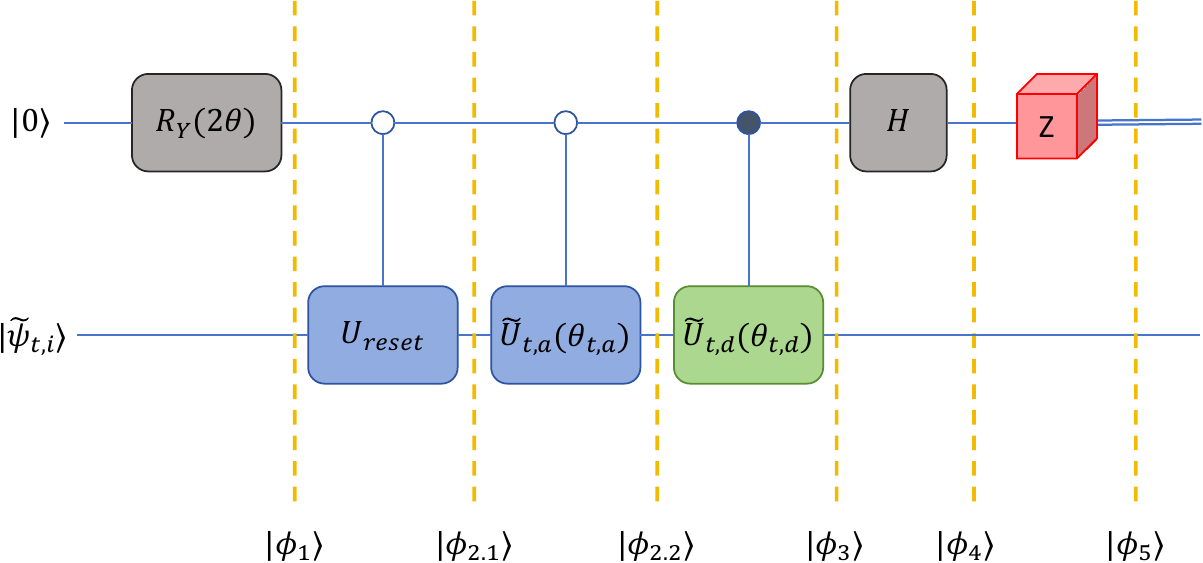}}
    \caption{Circuit construction for the state superposition in improved QuDDPM}
    \label{fig:improved_circuit}
    \end{center}
    \vskip -0.2in
\end{figure*}

The circuit construction for the state superposition in the improved QuDDPM is illustrated in Fig.~\ref{fig:improved_circuit}. Based on the original data qubits, we introduce only a single additional control qubit $c$, implementing the improved QuDDPM via controlled operations and post-selection processing. The specific procedure is outlined below:

\textbf{Step 1: Applying weights}: Apply a single-qubit rotation gate $R_Y(2\theta)$ to the control qubit $c$. At this point, the probability amplitudes of the control qubit dynamically correspond to the weight coefficients of the two paths, $\cos(\theta)$ and $\sin(\theta)$:
$$\vert{}\Phi_1\rangle = \left( \cos(\theta)\vert{}0\rangle + \sin(\theta)\vert{}1\rangle \right) \otimes \vert{}\tilde{\psi}_{t,i}\rangle.$$
This can be expanded into the block-superposition form:
$$\vert{}\Phi_1\rangle = \cos(\theta)\vert{}0\rangle \vert{}\tilde{\psi}_{t,i}\rangle + \sin(\theta)\vert{}1\rangle \vert{}\tilde{\psi}_{t,i}\rangle.$$

\textbf{Step 2: Generating $\vert{}\psi_a\rangle$}: The objective is to integrate the auxiliary qubits from the original improved QuDDPM into the data qubits and generate $\vert{}\psi_a\rangle$. Since the generation trajectory of the input state $\vert{}\tilde{\psi}_{t,i}\rangle$ is known, including both the circuit used to prepare the (approximate) Haar-random state and the circuit information for denoising it into $\vert{}\tilde{\psi}_{t,i}\rangle$, we can construct a reset circuit $U_{\text{reset}}$ such that $U_{\text{reset}}\vert{}\tilde{\psi}_{t,i}\rangle = \vert{}0\rangle^{\otimes n}$.We then apply the controlled-reset operator $\text{C}_0(U_{\text{reset}})$. When the control qubit is $\vert{}0\rangle$, the reset circuit is activated, forcing the data qubits back to the all-zero product state; when the control qubit is $\vert{}1\rangle$, the circuit remains inactive, leaving the original state unchanged. This yields: 
\begin{equation}
    \begin{split}
        \vert{}\Phi_{2.1}\rangle &= \cos(\theta)\vert{}0\rangle \left( U_{\text{reset}}\vert{}\tilde{\psi}_{t,i}\rangle \right) + \sin(\theta)\vert{}1\rangle \vert{}\tilde{\psi}_{t,i}\rangle\\
        &= \cos(\theta)\vert{}0\rangle \vert{}0\rangle^{\otimes n} + \sin(\theta)\vert{}1\rangle \vert{}\tilde{\psi}_{t,i}\rangle.
    \end{split}
\end{equation}
Subsequently, we apply the controlled unitary operator $\text{C}_0(\tilde{U}_{t,a})$, which yields
\begin{equation}
    \begin{split}
        \vert{}\Phi_{2.2}\rangle &= \cos(\theta)\vert{}0\rangle \left(\tilde{U}_{t,a}\vert{}0\rangle^{\otimes n}\right) + \sin(\theta)\vert{}1\rangle \vert{}\tilde{\psi}_{t,i}\rangle \\
        &= \cos(\theta)\vert{}0\rangle \vert{}\psi_a\rangle + \sin(\theta)\vert{}1\rangle \vert{}\tilde{\psi}_{t,i}\rangle.
    \end{split}
\end{equation}

\textbf{Step 3: Generating $\vert{}\psi_d\rangle$}: In the current step, we apply the controlled unitary operator $\text{C}_1(\tilde{U}_{t,d})$, which yields:
\begin{equation}
    \begin{split}
        \vert{}\Phi_3\rangle &= \cos(\theta)\vert{}0\rangle \vert{}\psi_a\rangle + \sin(\theta)\vert{}1\rangle \left(\tilde{U}_{t,d}\vert{}\tilde{\psi}_{t,i}\rangle\right) \\
        &= \cos(\theta)\vert{}0\rangle \vert{}\psi_a\rangle + \sin(\theta)\vert{}1\rangle \vert{}\psi_d\rangle.
    \end{split}
\end{equation}

\textbf{Step 4: Eliminating control qubit entanglement}: Apply a Hadamard gate to the control qubit $c$:
$$\vert{}\Phi_4\rangle = \frac{1}{\sqrt{2}}\cos(\theta)(\vert{}0\rangle+\vert{}1\rangle) \vert{}\psi_a\rangle + \frac{1}{\sqrt{2}}\sin(\theta)(\vert{}0\rangle-\vert{}1\rangle) \vert{}\psi_d\rangle,$$
Recombining the terms based on the $\vert{}0\rangle$ and $\vert{}1\rangle$ bases of the control qubit yields:
$$\vert{}\Phi_{4}\rangle = \vert{}0\rangle \otimes \left[ \frac{\cos(\theta) \vert{}\psi_a\rangle + \sin(\theta) \vert{}\psi_d\rangle}{\sqrt{2}} \right] + \vert{}1\rangle \otimes \left[ \frac{\cos(\theta) \vert{}\psi_a\rangle - \sin(\theta) \vert{}\psi_d\rangle}{\sqrt{2}} \right].$$

\textbf{Step 5: Generating $\vert{}\tilde{\psi}_{t-1,i}\rangle$}: Measure the control qubit $c$; when the measurement outcome corresponds to $\vert{}0\rangle$, we have:
$$\vert{}\tilde{\psi}_{t-1,i}\rangle=\vert{}\Phi_5\rangle = \kappa(\cos(\theta) \vert{}\psi_a\rangle + \sin(\theta) \vert{}\psi_d\rangle),$$
where $\kappa$ is the normalization factor.
\end{document}